\documentclass[letterpaper, 10 pt, conference]{ieeeconf}

\IEEEoverridecommandlockouts                              %

\usepackage[utf8]{inputenc}
\usepackage{blindtext}
\usepackage{tabularx}

\usepackage{amsthm}
\renewcommand\qedsymbol{$\blacksquare$}

\usepackage{mathtools}
\usepackage{dsfont}
\usepackage{amssymb}
\usepackage{cite}
\usepackage{xcolor}
\usepackage{tikz}
\usepackage{pgfplots}
\usepackage{svg}
\usepackage{adjustbox}
\usepackage{comment}
\usepackage{standalone}
\usepackage{calc}
\usepackage[acronym]{glossaries}
\usepackage{ifthen}
\newboolean{proofs}

\usepackage[
    colorlinks,
    linkcolor={blue!50!black},
    citecolor={blue!50!black},
    urlcolor={blue!80!black}]{hyperref}
    
\usepackage{multicol}
\usepackage{todonotes}
\presetkeys{todonotes}{inline}{}

\usepackage[capitalize]{cleveref}
\Crefname{equation}{Eq.}{Eqs.}
\Crefname{figure}{Fig.}{Figs.}
\Crefname{tabular}{Tab.}{Tabs.}
\Crefname{definition}{Def.}{Defs.}
\Crefname{section}{Sec.}{Sects.}
\Crefname{theorem}{Thm.}{Thms.}
\Crefname{condition}{Cond.}{Conds.}

\usepackage{graphicx}
\usepackage{booktabs}
\usepackage{threeparttable}
\usepackage{multirow}
\usepackage[ruled,vlined]{algorithm2e}
\usepackage{tikz, pgfplots, pgfplotstable}
\usetikzlibrary{
    positioning,
  	cd,
  	shapes,
  	math,
  	decorations.markings,
  	decorations.pathmorphing,
	decorations.pathreplacing,
  	arrows.meta,
  	calc,
  	fit,
  	quotes,
  }
\tikzcdset{arrow style=tikz, diagrams={>=To}}
\pgfplotsset{compat=1.15}

\let\labelindent\relax

\tikzstyle{block} = [draw, rectangle, minimum height=2em, minimum width=3em,thick]

\tikzstyle{blockdot} = [block, dotted,rounded corners=4, inner sep=-2pt]

\tikzstyle{blockfill} = [block,rounded corners=4, inner sep=-2pt,fill=blue!5!white]

\tikzstyle{every node}=[font=\footnotesize]

\makeatletter
\let\MYcaption\@makecaption
\makeatother
\usepackage[font=footnotesize]{subcaption}
\makeatletter
\let\@makecaption\MYcaption
\makeatother

\theoremstyle{definition}
\newtheorem{theorem}{Theorem}
\newtheorem{proposition}[theorem]{Proposition}

\newtheorem{claim}[theorem]{Claim}
\newtheorem{fact}[theorem]{Fact}

\newtheorem{lemma}[theorem]{Lemma}
\newtheorem{example}[theorem]{Example}
\newtheorem{remark}[theorem]{Remark}
\newtheorem{definition}[theorem]{Definition}
\newtheorem{property}[theorem]{Property}

\makeatletter
\def\@opargbegintheorem#1#2#3{\trivlist
   \item[]{\bfseries #1\ #2\ (#3)} \itshape}
\makeatother

\newcommand{\alphas}{\alpha^\star}

\newacronym{abk:br}{BR}{best response}

\newacronym{abk:cps}{CPS}{Cyber-physical system}
\newacronym{abk:dag}{DAG}{Directed Acyclic Graph}

\newcommand{\defeq}{\vcentcolon=}
\newcommand{\distance}{\delta}

\newcommand{\indicator}{\mathds{1}}

\newcommand{\lati}{\Tilde{\lambda}}

\newacronym{abk:marl}{MARL}{Multi-Agent Reinforcement Learning}

\newcommand{\muti}{\Tilde{\mu}}
\newcommand{\mutis}{\muti^\star\!}

\newcommand{\neighdegree}{h} %

\newacronym{abk:poset}{poset}{partially ordered set}

\newcommand{\proofref}[1]{\ifthenelse{\boolean{submission}}{}{The proof can be found in \cref{#1}.}}
\newcommand{\players}{\mathcal{A}}

\newcommand{\poa}{PoA}

\newacronym{abk:rl}{RL}{Reinforcement Learning}

\newcommand{\resourceset}{\mathcal{R}}
\newcommand{\resource}{r}
\newcommand{\load}{l}

\newcommand{\socialcost}{C}

\DeclarePairedDelimiter\floor{\lfloor}{\rfloor}
\DeclarePairedDelimiter{\ceil}{\lceil}{\rceil}

\newcommand{\game}{\mathcal{G}}
\newcommand{\strategyset}{\Gamma}
\newcommand{\strategy}{\gamma}

\newcommand{\strategyother}{\strategy_{-i}}

\newcommand{\cgstrategy}{\strategy^{\mathrm{cg}}}

\newcommand{\rescost}{\cost}

\newcommand{\cost}{J}
\newcommand{\costper}{\cost^\mathrm{per}}

\newcommand{\costcg}{\cost^\mathrm{cg}}

\newcommand{\reals}{\mathbb{R}}

\newcommand{\naturals}{\mathbb{N}}

\definecolor{nkcolor}{RGB}{90,180,90}
\newcommand{\nk}[1]{{\color{nkcolor}NK: #1}}

\newacronym{udg}{UDG}{Urban Driving Game}
\newacronym{ibr}{IBR}{iterative best response}
\newacronym[longplural={Nash Equilibria}, shortplural={NE}]{ne}{NE}{Nash Equilibrium}
\newacronym{svo}{SVO}{Social Value Orientation}
\newacronym{smarts}{SMARTS}{Scalable Multi-Agent Reinforcement Learning Training School}
\newacronym{marl}{MARL}{Multi-Agent Reinforcement Learning}

\usepackage{ifxetex}
\ifxetex
\usepackage[tracking=false,kerning=false,spacing=false]{microtype}
\usepackage[protrusion=true,tracking=false,kerning=false,spacing=false]{microtype}
\else
\usepackage[tracking=false,kerning=true,spacing=true]{microtype}

\fi
\usepackage{enumitem}
\setlist[enumerate]{itemsep=0mm,leftmargin=6mm,topsep=0mm}
\setlist[itemize]{itemsep=0mm,leftmargin=4mm,topsep=0mm}

\renewcommand{\baselinestretch}{0.958}

\setboolean{proofs}{false}
\bstctlcite{IEEEexample:BSTcontrol}

\begin{document}
\title{\LARGE \bf
How Bad is Selfish Driving?\\Bounding the Inefficiency of Equilibria in Urban Driving Games}
\author{Alessandro Zanardi$^{*}$, Pier Giuseppe Sessa$^{*}$, Nando K\"aslin, Saverio Bolognani, Andrea Censi, Emilio Frazzoli
\thanks{
$^{*}$Equal contribution.
}
\thanks{
This work was supported by the Swiss National Science Foundation under NCCR Automation, grant agreement 51NF40\_180545.}%
}

\maketitle
\begin{abstract}
We consider the interaction among agents engaging in a driving task and we model it as general-sum game.
This class of games exhibits a plurality of different equilibria posing the issue of equilibrium selection.
While selecting the most efficient equilibrium (in term of social cost) is often impractical from a computational standpoint, in this work we study the (in)efficiency of \emph{any} equilibrium players might agree to play.
More specifically, we bound the equilibrium inefficiency by modeling driving games as particular type of congestion games over spatio-temporal resources. We obtain novel guarantees that refine existing bounds on the Price of Anarchy (\poa{}) as a function of problem-dependent game parameters.
For instance, the relative trade-off between proximity costs and personal objectives such as comfort and progress. 
Although the obtained guarantees concern open-loop trajectories, we observe efficient equilibria even when agents employ closed-loop policies trained via decentralized multi-agent reinforcement learning. 
\end{abstract}
\IEEEpeerreviewmaketitle

\addtolength{\belowcaptionskip}{-10pt}
\setlength{\textfloatsep}{10pt plus 1.0pt minus 2.0pt}

\section{Introduction}
While autonomous vehicles begin to be deployed around the world, it became evident that they often still miss the magic touch to seamlessly integrate with other road users~\cite{Wang2022SocialPerspectives}.
This has sparked a noticeable research interest toward the interactive nature of the driving task~\cite{Albrecht2018b,Wang2022SocialPerspectives,Crosato2022Interaction-awareOrientation}. 
To this end, \emph{game-theoretical} notions have been integrated in motion planning algorithms~\cite{Di2020}, in learning policies~\cite{Sadigh,Peters2022LearningGames} and, more in general, when explicitly reasoning about others' reactive behavior~\cite{Stefansson2019}.\looseness=-1  

Arguably, the hardness of driving interactions is to coordinate on a certain equilibrium~\cite{Wang2022SocialPerspectives} -- who goes first when resources are contended.
Under mild assumptions, it has been shown that there actually exist certain equilibria that shall be preferred in terms of social efficiency~\cite{Zanardi2021} or in terms of cost sharing~\cite{LeCleach2022ALGAMES:Games}.
At the same time, big strides forward have been made for game-theoretical planners that have local guarantees of convergence~\cite{LeCleach2022ALGAMES:Games,Fridovich-Keil2019,Sadigh}.
Combining these two aspects, our work is motivated by the following question: \emph{How inefficient can an equilibrium be compared to another?}
An answer would have several implications ranging from the problem of equilibrium selection~\cite{Harsanyi1988} to the importance of global vs local, centralized vs decentralised solutions.

We consider the class of urban driving games and study their efficiency, i.e., the cost of their equilibria with respect to the social optimum.
Under minor modeling assumptions, we show that it is possible to derive analytical bounds for their inefficiency. 
The resulting bound is a function of the relative importance between personal objectives (e.g., a comfort cost that depends only on the agent's trajectory) and joint ones (e.g., a proximity cost that depends on the joint trajectories of the players). In addition, it depends on a problem-dependent parameter which represents the agent's sensitiveness to the number of nearby vehicles.
A satisfactory bound implies that the agents can be self-interested without having to estimate others' degree of cooperativeness~\cite{Schwarting2019a}; and that there is no need for global coordination since decentralized and local solutions would still achieve a satisfactory overall cost for the system.
\newlength{\mycolumn}
\newlength{\myleftcut}
\newlength{\myrightcut}
\newlength{\mybelowcut}
\newlength{\mytopcut}
\newlength{\myinflated}
\setlength{\mycolumn}{\columnwidth}

\setlength{\myleftcut}{.1cm}
\setlength{\mybelowcut}{.3cm}
\setlength{\myrightcut}{.1cm}
\setlength{\mytopcut}{.1cm}

\begin{figure}[t]
    \centering
    \begin{subfigure}{\mycolumn}
			\adjincludegraphics[width=\mycolumn,Clip={\myleftcut \mybelowcut \myrightcut \mytopcut},clip]{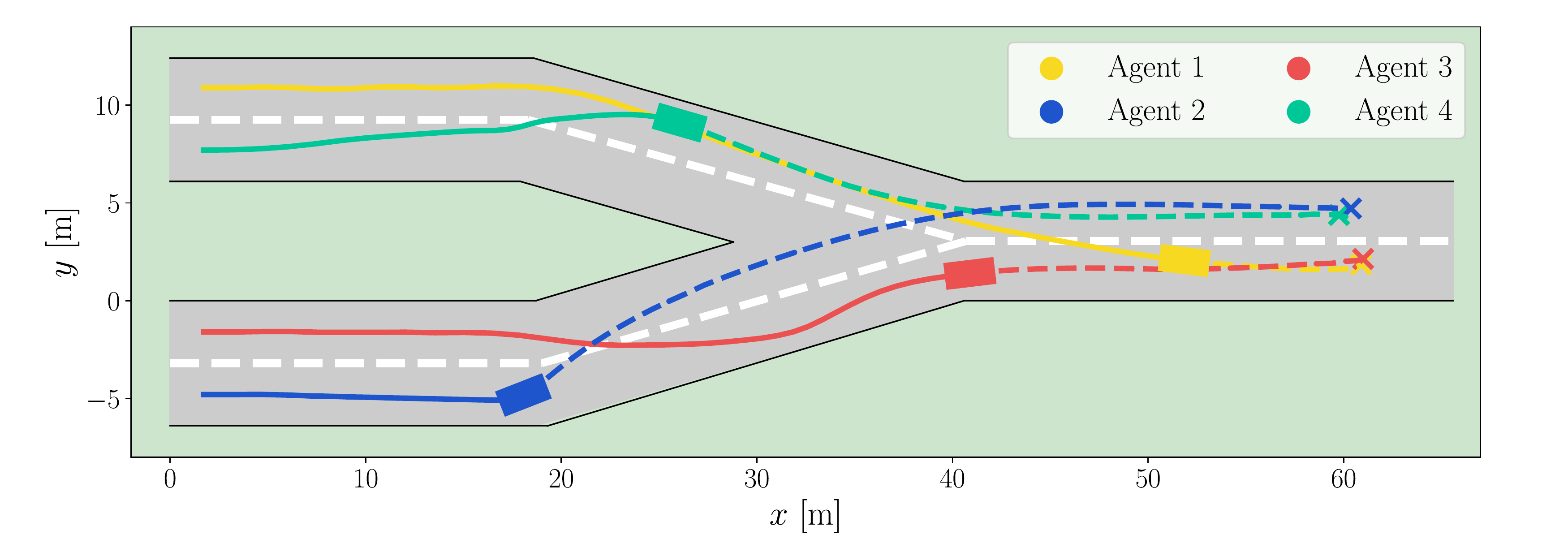}
    \caption*{}
    \label{fig:sub3}
    \end{subfigure}
    \begin{subfigure}{\mycolumn}
      \adjincludegraphics[width=\mycolumn,Clip={\myleftcut \mybelowcut \myrightcut \mytopcut},clip]{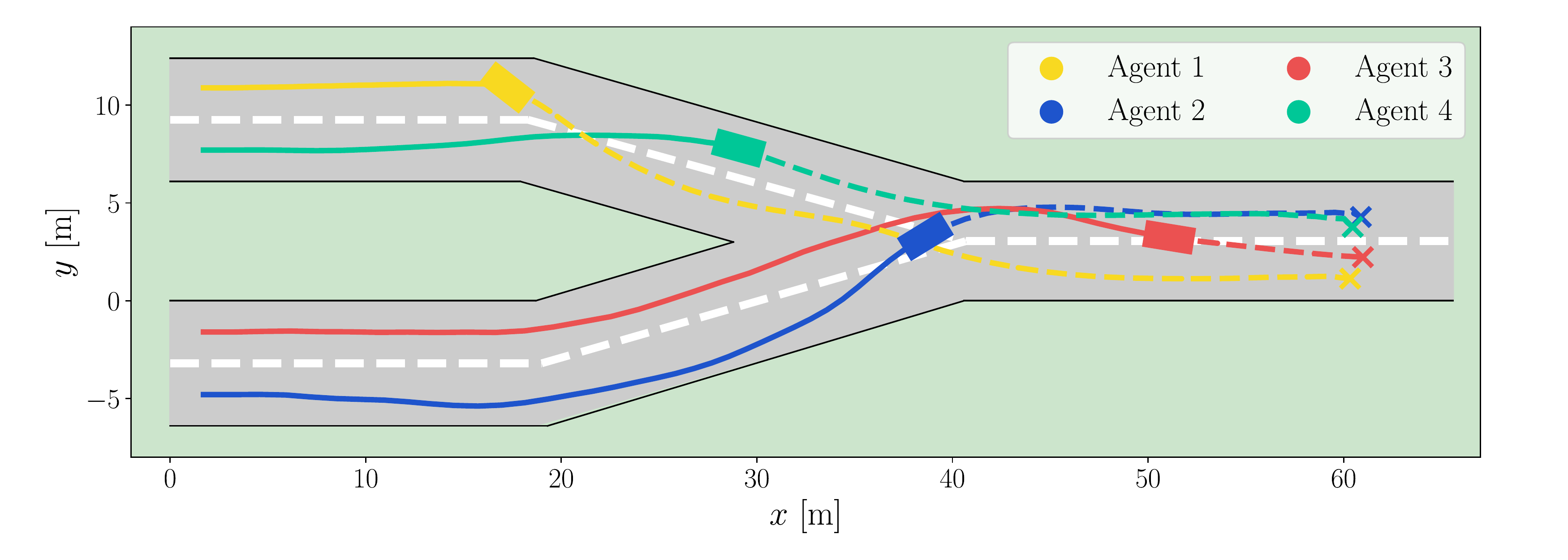}
    \caption*{}
    \label{fig:sub4}
    \end{subfigure}
\caption{For agents engaging in the driving task there exist many topologically different equilibria. The figure shows two distinct examples of learned Nash Equilibrim policies. Some would result in better overall costs for the individual agents but would require to find global solutions which are often impractical in these cases. We study and bound in terms of Price of Anarchy the inefficiency arising in these type of games.}
\label{fig:driving_game}
\end{figure}

\subsection{Related Work}
Many works in the last years modeled driving interactions as a general-sum game~\cite{Toghi2021SocialDriving, Zanardi2021}.
And most--if not all--of the devised solution methods provide guarantees only for local convergence to Nash Equilibria; examples range from iterative quadratic approximations~\cite{Fridovich-Keil2019}, augmented Lagrangian methods~\cite{LeCleach2022ALGAMES:Games} and ``Newtonesque" methods~\cite{Di2019NewtonsGames}.
Since there exists both a continuum of solutions but also qualitatively different class of solutions (in the sense of topologically different solutions~\cite[Sec. 3.3]{Wang2022SocialPerspectives}), one cannot ignore the problem of equilibrium selection. For some methods the (local) choice is embedded in the method, converging for example to Generalized Nash Equilibria~\cite{LeCleach2022ALGAMES:Games}. In other mixed context the choice is dictated by human drivers. In~\cite{Peters2020} for instance, autonomous vehicles keep a belief over the possible equilibria in a bid to favor the ones preferred by others.
Differently from these works, we study the inefficiency that \emph{any} equilibrium could have.

The study of games' inefficiency finds a considerable body of literature since the pioneering work in~\cite{Koutsoupias2009Worst-caseEquilibria}. 
Efficiency guarantees are often expressed in terms of \poa{} which quantifies the ratio between the social cost of the worst-performing equilibrium with that of the socially optimal outcome. 
\poa{} bounds have been derived for specific classes of games such as congestion games~\cite{Roughgarden2002HowRouting,Aland2011Exact}, utility games~\cite{Vetta2002NashAuctions, Sessa2019b}, and smooth games~\cite{Roughgarden2015}, exploiting various structures. 
However, to the best of our knowledge, PoA in the urban driving setting have not been studied in the literature. 
In this work, we formulate driving games as a particular type of congestion games and employ the PoA bounding techniques of~\cite{Aland2011Exact}. However, differently from~\cite{Aland2011Exact} we generalize and refine the obtained guarantees exploiting the specific driving games' cost structure which consists of joint but also personal objectives.\looseness=-1
\subsection{Contribution}
We consider the problem of bounding the inefficiency of equilibria that emerge in driving games, quantified via the notion of PoA.
To this end,
\begin{itemize}
    \item We formally show that driving games can be naturally modeled as a particular type of \emph{congestion games}, where the agents compete for spatio-temporal resources. This allows us to leverage existing PoA bounds for congestion games and apply them to our class of driving games. 
    \item We further refine the efficiency guarantees by exploiting the specific cost structure of driving games. We derive a novel and improved efficiency bound which depends on the relative importance between personal and joint costs. The obtained results can be of broader interest since they apply to general congestion games with added personal costs and, to the best of our knowledge, they constitute the first PoA bounds for driving games.\looseness=-1 
    \item We conduct an experimental case study to evaluate the inefficiency of several driving scenarios. We compute equilibrium driving policies via multi-agent reinforcement learning and utilize a systematic approach to empirically approximate the associated PoAs. Conforming with our intuition, the computed equilibria display a high efficiency in all the considered scenarios and the resulting PoAs are within the derived, albeit conservative, bounds. 
\end{itemize}

\section{Preliminaries}\label{sec:preliminaries}
\subsection{Urban Driving Games}
We consider the class of (urban) driving games akin to~\cite{Zanardi2021}.
They are a particular subclass of general-sum games with few peculiarities. 
Most importantly, the cost-structure in a driving game allows to distinguish between \emph{joint} and \emph{personal} costs. 
Joint costs depend on the state and actions of all the players, e.g., proximity. 
Personal costs instead, depend only on the states and actions of a specific player, e.g., a comfort objective penalizing large accelerations. 
Furthermore, the resulting game often enjoys the favorable structure of being a potential game~\cite{Kavuncu2021a,Liu2022PotentialDriving}, which in turn, guarantees convergence of \emph{better}-response schemes and, in some cases~\cite{Zanardi2021}, social efficiency of global minima. 

For the provided analytic results we consider \emph{open-loop} strategies, where the players commit to the whole trajectory. 
More formally, we consider a driving game defined by the tuple $\game = \langle \players, \{ \strategyset_i\}, \{J_i\} \rangle_{i\in\players}$., where $\players$ is a finite set of players, $\strategyset_i$ is a discrete set of dynamically feasible trajectories, and $J_i$ is the cost structure for a player. 
In the remaining, we denote without subindices the joint quantities, while the subindex specifies if a quantity is peculiar to a subset of the players; e.g., $-i$ reads ``all but Player $i$''. Thus, we denote the trajectory choice of Player $i$ as $\strategy_i \in \strategyset_i$, whereas $\strategy \in \prod_{i\in\players} \strategyset_i $ denotes the joint trajectories of all players.

In a driving game, the cost structure $J_i$ of each player penalizes -- as a first priority objective -- colliding trajectories. 
Second, whenever game outcomes are not colliding, it typically penalizes distances from neighbouring players (i.e., a \emph{proximity} cost) and other personal objectives (i.e., comfort, acceleration, time, etc.). This principle has been modeled with a lexicographic ordered cost in~\cite{Zanardi2021}, and with optimization constraints in other works~\cite{Dreves2018,LeCleach2022ALGAMES:Games}.

In this work, we restrict the possible coupling costs among the players to the ones involving distance.
Therefore, the overall cost for a player has the form $J_i(\strategy) = J_i^\text{prox}(\strategy) + J_i^\text{per}(\strategy_i)$ where the first term is a proximity cost--collision at the limit--and the second term is a personal cost.
We further specify the allowed proximity costs to have two properties:
\begin{enumerate}[label=(\roman*)]
    \item To be integrable over the trajectory;
    \item To be monotonically increasing as the distance decreases.
\end{enumerate}
More formally, given a pair of players' trajectories $\strategy_i$ and $\strategy_j$, we define $\distance(\strategy_i, \strategy_j, t) $ as the spatial distance between $\strategy_i$ and $\strategy_j$ at time $t$. Then, the proximity costs $J^\text{prox}_i$ must satisfy the following property.
\begin{property}\label{property}
For any player~$i$ and others' trajectories $\strategy_{-i}$, consider any pair of feasible trajectories $\strategy_i, \strategy_i'\in \strategyset_i$.
Then, if $\distance(\strategy_i' ,\strategy_j, t) \leq \distance(\strategy_i ,\strategy_j, t),  \forall t, \forall j \neq i$, it must hold $J_j^\text{prox}(\strategy_i', \strategy_{-i}) \geq J_j^\text{prox}(\strategy_i, \strategy_{-i})$, $\forall j \in \players$.
\end{property}
Intuitively,~\Cref{property} ensures that for a unilateral deviation of player~$i$, the proximity cost of every player (including $i$) increases as player~$i$ chooses trajectories that are spatially closer to the others. 
Overall,~\cref{property} encompasses many possible choices of proximity costs such as 
$J_i^\text{prox}(\strategy) = - \sum_{j\neq i}  \sum_t \distance(\strategy_i, \strategy_j, t)^\alpha$,
for a given degree coefficient $\alpha>0$, or of the kind $J_i^\text{prox}(\strategy) =  \sum_{j\neq i} \sum_t \frac{1}{\distance(\strategy_i, \strategy_j, t)^\alpha}$. Additionally, ~\cref{property} is still valid whenever only distances within a certain thresholds are penalized as e.g. in
\begin{equation}
    J_i^\text{prox}(\strategy) =  \sum_{j\neq i} \sum_t \begin{cases}
            (\distance_s - \distance(\strategy_i, \strategy_j, t))^\alpha \quad &\text{if $\distance(\strategy_i, \strategy_j, t) < \distance_s$,}\\
            0 &\text{otherwise,}
        \end{cases}
        \label{eq:clear_cost_ds}
\end{equation}
where $\alpha > 1$ and $\distance_s$ is some safety distance.
\subsection{Game Equilibria and Efficency}
We consider \glspl{ne} as the solution concepts of driving games, defined as follows.
\begin{definition}[Nash Equilibrium]
    A trajectory profile $\strategy$ is a (pure) Nash equilibrium (NE) if $\forall i \in \players$, $\forall \strategy_i' \in \strategyset_i$
    \begin{equation}
        \cost_i(\strategy) \leq \cost_i(\strategy_i',\strategyother)
        \label{eq:NE_condition}.
    \end{equation}
    \label{def:PNE}
    We denote the set of NE outcomes as $\strategyset_\textrm{NE} \subseteq \strategyset$.
\end{definition} 
Different approaches have been proposed for computing NE trajectories, e.g.,~\cite{Fridovich-Keil2019,LeCleach2022ALGAMES:Games,Di2019NewtonsGames}.
Unfortunately, however, the fact that NE can be computed does not tell us anything about their quality (i.e., their \emph{efficiency}).
A common way to quantify the efficiency of a game outcome is by measuring its social cost:
\begin{definition}[Social Cost]
The \emph{social cost} $\socialcost(\strategy)$ associated with a specific outcome $\strategy$ of a game is defined as the sum over all the individual player costs:
    \begin{equation}
        \socialcost(\strategy) \defeq \sum_{i\in\players}\cost_i(\strategy).
    \end{equation}
\end{definition}
A game is more efficient if it results in a lower social cost. In the context of urban driving, efficient outcomes are usually represented by trajectories allowing the players to reach their goals, at a safe distance, and with minimal total consumption, e.g., of time, acceleration, fuel, etc.

Because the players are self-interested, they aim at minimizing their individual costs $J_i$ rather than $C$ causing inefficiency for the game.
To measure the game \emph{inefficiency} we adopt the widely used notion of \poa{}~\cite{Koutsoupias2009Worst-caseEquilibria}.
\begin{definition}[Price of Anarchy]\label{def:poa}
The \textit{Price of Anarchy} (\poa{}) is the ratio between the highest social cost at a NE and the lowest social cost overall:
    \begin{equation}
        \textrm{\poa} = \frac{\max_{\strategy \in \strategyset_\textrm{NE}} \socialcost(\strategy)}{\min_{\strategy \in \strategyset} \socialcost(\strategy)} \in [1, \infty).
    \end{equation}
\end{definition}
In general, providing PoA bounds is a hard task since these must clearly depend on the specific game and players' costs structure.
In the following, we show that driving games can be naturally modeled as a particular type of \emph{congestion games} and suitable PoA bounds can be derived inheriting -- and refining -- existing guarantees for such specific games' structure.\looseness=-1

\section{Driving Games as Congestion Games}
\label{sec:dg_as_cg}
Inspired by the robotics literature, we look at the class of driving games as agents competing for common resources~--~in this case, portions of the road at specific time instances.
In this spirit, we show that the games presented in~\Cref{sec:preliminaries} can be (re)modeled as congestion games which preserve the same key properties, allowing us to derive inefficiency bounds.

Congestion games were first introduced in~\cite{Rosenthal1973} as games where the agents' strategy corresponds to selecting a subset of the available resources. The use of each resource is penalized by a monotonic load function such that, the more players select that resource, the more cost they incur.
Hence, the total cost for a player results in the sum over the load costs of the selected resources. 
We refer to~\cite{Hespanha2017} for a more pedagogical presentation.\looseness=-1
\subsection{Congestion Game Formulation}
In the driving setting at hand, we consider the finite set of resources given by a discretization of the road in both space (i.e. a 2D grid) and time. 
We denote the set of spatio temporal resources as $\resourceset$. Then, 
each trajectory $\strategy_i$ can be mapped to a corresponding strategy $\cgstrategy_i \subseteq \resourceset$ by its spatio-temporal occupancy.
Since we consider deterministic trajectories, each agent either uses a resource or it does not. In other words, the load an agent can put on a resource is binary. Hence, the resulting load on resource $r$ is defined as $\load_{\resource}(\cgstrategy)=\sum_{i\in\players} \indicator_{\resource\in\cgstrategy_i} \in\naturals$, where $\indicator_{\bullet}$ is the indicator function. 
Each resource then has a specific load-dependent cost function $\cost_\resource:\naturals \to \reals$. 
For the purpose of this work, we restrict $\cost_\resource$ to be a polynomial with non-negative coefficients. 
The cost which an agent incurs is then constructed as:
\begin{equation*}
    \costcg_i(\cgstrategy) = \sum_{\resource\in \cgstrategy_i} \cost_\resource (\load_\resource(\cgstrategy)).
\end{equation*}
Intuitively, to minimize the above congestion cost, agents are encouraged to choose non-overlapping trajectories and thus the above formulation models -- as a first approximation -- driving games' preferences. However, it is a very crude approximation since it does not discriminate among non-overlapping trajectories and thus cannot fully model proximity costs.
In the following, we show that augmenting the resource set along a new dimension that we name ``proximity dimension" allows us to model various proximity costs that respect~\Cref{property}. 
The inclusion of the personal costs is instead straightforward (see end of this section) as already shown in the literature~\cite{Le2019Congestion}.

\begin{figure}[t!]
	\centering
	\vspace{-1.5em}
	\includegraphics[scale=0.9]{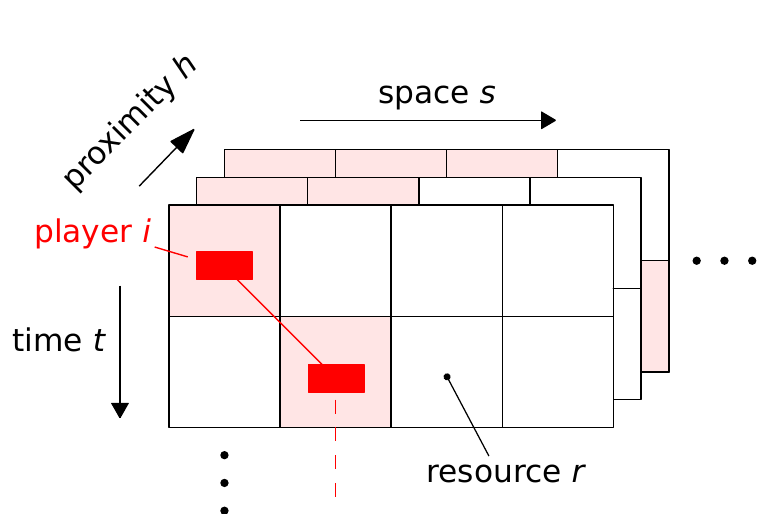}
	\caption{The set of possible resources follows from a discretization in space, time, and proximity levels. The figure depicts an example of the first two time steps of a ``one-dimensional'' road with three proximity levels, i.e., $H=3$.  The resources that are used by player $i$ are shaded in red and grow along the proximity dimension.}
	\label{fig:atomicCG_neigh}
\end{figure}
\paragraph*{Proximity levels} 
On top of the discretization in space and time, we further consider a \emph{proximity dimension} $\neighdegree \in \{0, \ldots, H-1\}$ so that a ``copy'' of the spatio-temporal resources exists for each proximity level $\neighdegree$, as shown in~\Cref{fig:atomicCG_neigh}. 
Intuitively, the trajectory of an agent progressively inflates its occupancy along this dimension, such that, at the larger proximity levels, resources can overlap even if the trajectories do not physically overlap. This allows to model proximity costs. 
More precisely, consider any given time $t$. Then, the spatial resources occupied at each proximity level are determined in the following manner: 
At the first level ($\neighdegree=0$), the trajectory $\strategy_i$ uses only the resources associated to its physical occupancy (we denote them as $[\cgstrategy_i(t)]_0$).
Then, for each successive level $h$, $\strategy_i$ uses the spatial resources that are within a neighborhood (i.e., a ball) of a given radius $\rho_\neighdegree$ around the agent's position $\strategy_i(t)$. We denote such resources as $[\cgstrategy_i(t)]_h$ and assume for simplicity that $\rho_\neighdegree > \rho_{\neighdegree-1}$.  Hence, by letting $\cost_\neighdegree(\cdot)$ represent the polynomial cost functions associated to resources of proximity level $\neighdegree$, the agents' proximity cost can be written as:
\begin{equation}\label{eq:cg_proximity_cost}
    \costcg_i(\cgstrategy) = \sum_{t=0}^{T-1} \sum_{\neighdegree = 0}^{H-1} \sum_{\resource\in [\cgstrategy_i(t)]_\neighdegree} \cost_\neighdegree (\load_{\resource}(\cgstrategy)).
\end{equation}

In~\cref{fig:atomicCG_neigh}, an illustrative example is shown with two additional proximity levels (i.e., $H=3$). According to this formulation, resources at higher levels of proximity can overlap even when players are driving at a certain distance allowing for penalization of unsafe driving maneuvers. Moreover, the use of different level-specific polynomial costs $J_\neighdegree$ allows to adjust the relative importance among proximity levels $\neighdegree$, to get much higher costs for lower levels of proximity. This allows us to model -- via the congestion game's costs of~\eqref{eq:cg_proximity_cost} -- different types of proximity costs similar to the ones presented in~\eqref{eq:clear_cost_ds}. We illustrate such expressiveness in the following numerical examples, where we consider polynomials $\cost_\neighdegree$ with different coefficients and degree.

\begin{figure}[t]
	\centering
	\hspace{-1.0em}
\includegraphics[width=.5\textwidth]{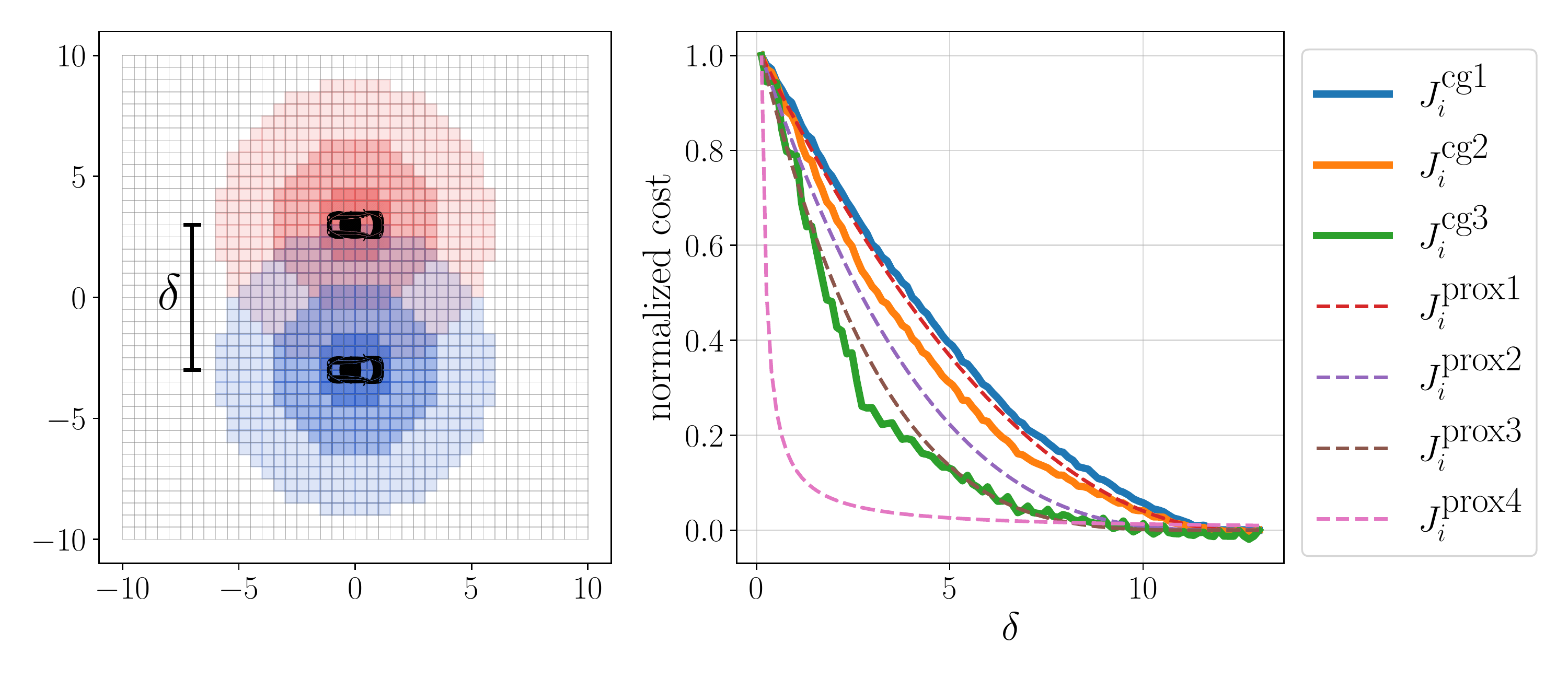}
	\caption{The left figure shows a visualization of the two-player game of~\Cref{example:twocars} at a single time instance $t$. The right figure depicts different (normalized) congestion game costs $\costcg_i$ defined in~\Cref{example:twocars}, as a function of the distance $\distance$ between the two cars. The costs differ in terms of relative weight between different proximity levels (through the polynomials $J_\neighdegree$). Other common proximity costs choices defined in~\Cref{example:twocars} are also shown.}
	\label{fig:resources_cost_viz}
\end{figure}

\begin{example}\label{example:twocars}
We consider two cars at distance $\distance$ from each other. For simplicity, we consider a fixed time step and we are only interested in how their cost $\costcg_i$ changes as a function of their distance. We take three proximity levels ($H=3$) but for interpretability, we represent all of them in the same grid and color the respective used resources with different shades, as depicted in~~\cref{fig:resources_cost_viz} (left plot). The space is discretized with $0.5 \times 0.5$m grid cells and neighborhoods are Euclidean balls with radii $\rho_0=1.5$m, $\rho_1=3.5$m, and $\rho_2=6$m, respectively. We consider polynomial resource costs of the form $\cost_\neighdegree (x) \defeq a_\neighdegree \cdot x^{d_\neighdegree}$ where $x$ is the total load and $a_\neighdegree$ and $d_\neighdegree$ are parameters that we set as follows. We fix $d_\neighdegree = 2, \forall \neighdegree$ and consider three configurations for the weights $\{a_0, a_1, a_2\}$ to set the relative importance between proximity levels: $\{1.0, 1.0, 1.0\}$, $\{.9, .4, .2\}$, and $\{1., .1, .02\}$, respectively. This leads to the three congestion game cost curves $J^\text{cg1}_i,J^\text{cg2}_i$, and $J^\text{cg3}_i$ depicted in~\cref{fig:resources_cost_viz} (right plot).
As visible, the more (relative) weight is given to low proximity levels (i.e., configurations $J^\text{cg2}_i$, and $J^\text{cg3}_i$) the steeper the cost decreases with $\distance$. Hence, the choice of weights  $\{a_\neighdegree\}_{h=1}$ can model different types of proximity costs allowing to control such steepness.
For comparison, we also compare the obtained costs with common choices of proximity costs discussed in Section~\ref{sec:preliminaries} (which are analytical functions of $\distance$): $J^\text{prox1}$, $J^\text{prox2}$, $J^\text{prox3}$ which are computed as in~\eqref{eq:clear_cost_ds} with $\distance_s = 12.5$ and $\alpha=2$, $3$, and $4$, respectively, and $J^\text{prox4}(\distance) = \distance^{-1}$.\looseness=-1
\end{example}

In the next example, we illustrate the role of the polynomial degrees $d_\neighdegree$. Indeed, it can be verified that the the normalized costs of~\Cref{fig:resources_cost_viz} are not influenced by the degree $d_h$ (since with only two agents higher degrees would result only in a higher offset and constant scaling factor). This is not the case when there are more than two agents -- as illustrated in the next example --  where $d_\neighdegree$ controls the sensitivity with respect to the number of other agents that co-occupy the same resource.

\begin{figure}[t]
	\centering
	\hspace{-1.0em}

\includegraphics[width=.46\textwidth]{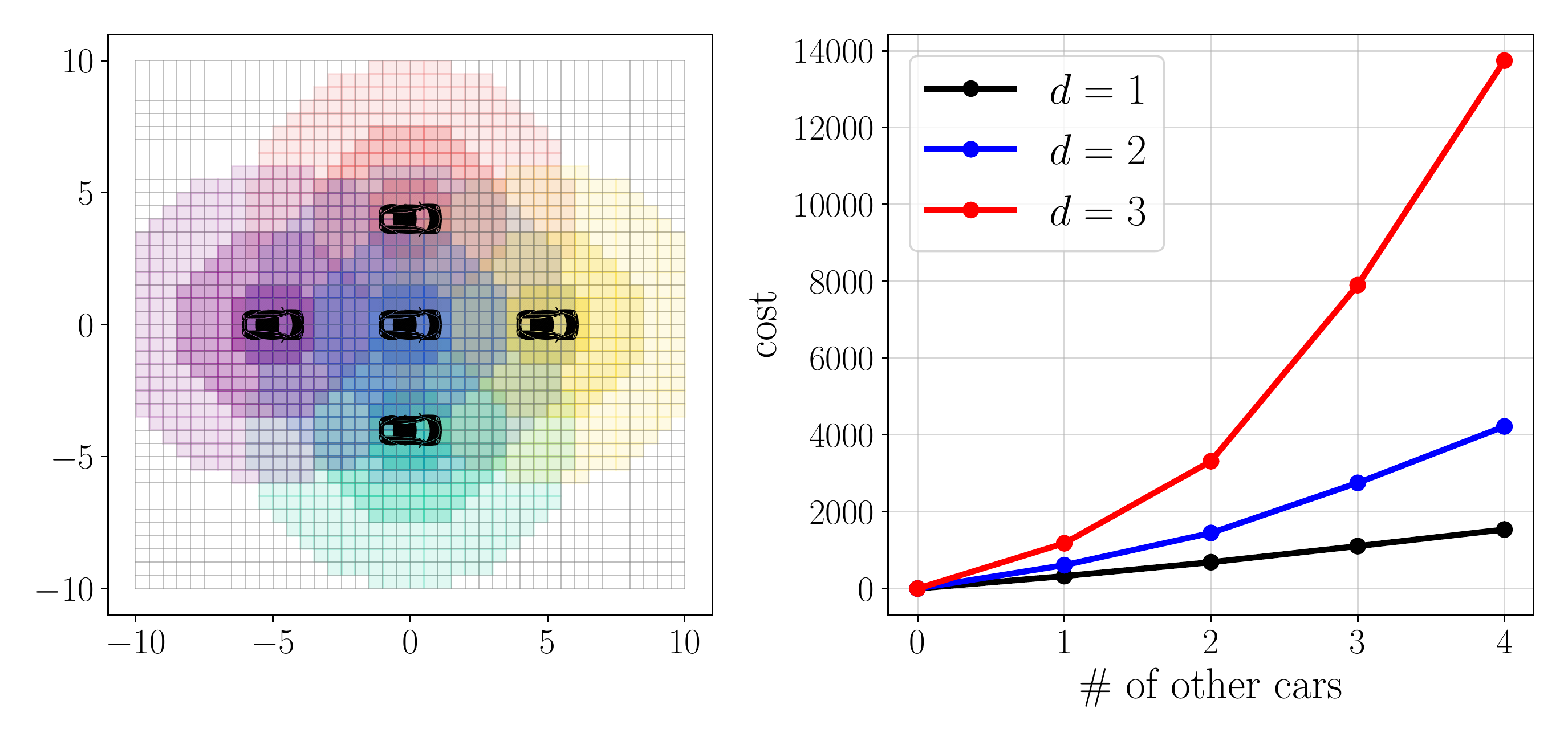}
	\caption{The left figure depicts the multi-player game configurations of~\Cref{example:3cars}. The right figure depicts the congestion game cost (up to constant offsets) of the car in position $(0,0)$ when only 1, 2, 3, or 4 of the other cars are in the game. The different cost curves correspond to using polynomial load functions $J_\neighdegree$, defined in~\Cref{example:3cars}, with degree 1, 2, or 3, respectively.\looseness=-1} %
	\label{fig:three_cars}
\end{figure}

\begin{example}\label{example:3cars}
Consider the same setup of~\Cref{example:twocars} where now more than 2 cars are on the road as displayed in~\Cref{fig:three_cars} (left plot). We are interested in computing the cost $\costcg_i$ of player~$i$ (which represents the car in position $(0,0)$) as a function of the number of other cars driving nearby. Hence, we consider the driving scenarios in which only 1, 2, 3, or 4 of the other cars are present in the game. Similar to~\Cref{example:twocars}, we use polynomial load functions $\cost_\neighdegree (x) \defeq a_\neighdegree \cdot x^{d}$. Now, we fix $\{a_0, a_1, a_2\}= \{.9,.4,.2\}$ and consider three cost configurations %
defined by degrees $d=1$, $d=2$ and $d=3$, respectively. The corresponding costs, as a function of the number of other cars are plotted in~\Cref{fig:three_cars} (right plot), removing constant offsets. As visible, the higher the polynomials' degree the steeper the cost as a function of nearby cars. Hence, the polynomial degree $d$ controls how much player~$i$ is sensible to the number of neighbouring players. We note that this is an extra degree of freedom that follows from our congestion game modeling and is not present, e.g., in the proximity costs of~\eqref{eq:clear_cost_ds}. Indeed, in~\eqref{eq:clear_cost_ds} (as well as in most considered driving game formulations~\cite{Zanardi2021}) the proximity cost of player~$i$ grows linearly with the number of other players (i.e., $d=1$).
\end{example}

 Examples~\ref{example:twocars} and~\ref{example:3cars} show that the congestion game cost formulation of~\eqref{eq:cg_proximity_cost} can naturally model a wide spectrum of proximity costs, with different degrees of freedom, thus serving as a good model for driving preferences. In addition, we more formally show that such cost formulation satisfy the driving games'~\Cref{property} defined in the previous section.
\begin{proposition}
The congestion game cost $\cost_i^\text{cg}$ defined in~\eqref{eq:cg_proximity_cost} satisfies~\Cref{property}.
\end{proposition}
\begin{proof}

Fix strategies $\strategy_{-i}$ and consider trajectories $\strategy_i$ and $\strategy_i'$ for player~$i$. Then, at each time $t$ and proximity level $\neighdegree$ there is a one-to-one mapping between the resources $r \in [\cgstrategy_i(t)]_\neighdegree$ and the ones in $[{\strategy_i'}^\text{cg}(t)]_\neighdegree$, since the latter are simply obtained by translating the occupancy of player~$i$ from position $\strategy_i(t)$ to $\strategy_i'(t)$. Let $\Gamma: \mathcal{R} \rightarrow \mathcal{R}$ be such a mapping. Assume now that $\distance(\strategy_i' ,\strategy_j, t) \leq \distance(\strategy_i ,\strategy_j, t),  \forall j \neq i,  \forall t$. Then, at each time $t$ and neighborhood $h$,  trajectory $\strategy_i' $ utilizes resources that have more overlap with other players, compared to when trajectory $\strategy_i$ is used. That is, $\forall t \in [T], h \in [H]$ and $\forall r \in [\cgstrategy_i(t)]_\neighdegree$, $l_{\Gamma(r)}({\strategy_i'}^\text{cg}, \cgstrategy_{-i}) \geq l_r(\cgstrategy_i, \cgstrategy_{-i})$. Moreover, $\forall r \in [\cgstrategy_j(t)]_\neighdegree$, $l_r({\strategy_i'}^\text{cg}, \cgstrategy_{-i}) \geq l_r(\cgstrategy_i, \cgstrategy_{-i})$. For player~$i$ this implies that 
$\costcg_i({\strategy_i'}^\text{cg}, \cgstrategy_{-i}) = \sum_{t=0}^{T-1} \sum_{\neighdegree = 0}^{H-1} \sum_{\resource\in [{\strategy_i'}^\text{cg}(t)]_\neighdegree} J_\neighdegree (\load_{\resource}({\strategy_i'}^\text{cg}, \cgstrategy_{-i}))  =  \sum_{t=0}^{T-1} \sum_{\neighdegree = 0}^{H-1} \sum_{\resource\in [{\strategy_i}^\text{cg}(t)]_\neighdegree} J_\neighdegree (\load_{\Gamma(\resource)}({\strategy_i'}^\text{cg}, \cgstrategy_{-i})) \geq  \sum_{t=0}^{T-1} \sum_{\neighdegree = 0}^{H-1} \sum_{\resource\in [\cgstrategy_i(t)]_\neighdegree} J_\neighdegree (\load_{\resource}(\cgstrategy_i, \cgstrategy_{-i}))=  \costcg_i(\cgstrategy_i, \cgstrategy_{-i}) $, where the last inequality follows since $l_{\Gamma(r)}({\strategy_i'}^\text{cg}, \cgstrategy_{-i}) \geq l_r(\cgstrategy_i, \cgstrategy_{-i})$ and the load functions $\cost_\neighdegree$ have non-negative coefficients. Moreover, for any other player $j\neq i$ it holds  $\costcg_j({\strategy_i'}^\text{cg}, \cgstrategy_{-i})-  \costcg_j({\cgstrategy_i}, \cgstrategy_{-i}) = \sum_{t=0}^{T-1} \sum_{\neighdegree = 0}^{H-1} \sum_{\resource\in [{\cgstrategy_j}(t)]_\neighdegree} \big[J_\neighdegree (\load_{\resource}({\strategy_i'}^\text{cg}, \cgstrategy_{-i})) - J_\neighdegree (\load_{\resource}({\cgstrategy_i}^\text{cg}, \cgstrategy_{-i}))\big] \geq 0$, due to monotonicity of $\cost_\neighdegree$. 
\end{proof}
\paragraph*{Personal cost}
Following similar works on congestion games~\cite{Le2019Congestion} we add a personal cost term to the overall cost for player~$i$ choosing strategy $\strategy_i$:
\begin{equation}
    \cost_i(\strategy) = \costcg_i(\cgstrategy) + \costper_i(\strategy_i).
\end{equation}
First, note that it was shown that such a game retains its exact potential game status, meaning a pure Nash Equilibrium is still guaranteed to exist~\cite{Le2019Congestion}. 
Moreover, as we more formally show in the next section, the personal cost can only have a positive influence on the inefficiency bounds of the game.

\section{Inefficiency Bounds}\label{sec:ineff_bounds}
In this section we present the inefficiency bounds that follow from our congestion game modeling of driving games. 
First, we recall existing \poa{} bounds for congestion games with polynomial load cost functions~\cite{Aland2011Exact}.
We then show that these bounds can be refined thanks to the particular cost structure of driving games.

The following theorem follows from~\cite{Aland2011Exact} bounding the \poa{} as a function of the degree of the polynomial load functions. 
\begin{fact}\label{thm:poa_bounds_original}\cite[Theorem 4.1]{Aland2011Exact}
For a congestion games with polynomial resource load function with non-negative coefficients and degree at most $d\in\naturals$, it holds 
\begin{equation}
   \text{\poa{}} \leq  \frac{(k+1)^{2d+1}-k^{d+1}(k+2)^d}{(k+1)^{d+1}-(k+2)^d+(k+1)^d-k^{d+1}},
    \label{eq:cg_poa_upper_bound}
\end{equation}
where $k\defeq\floor{\Phi_d}$ and $\Phi_d$ is the positive real solution to $(x+1)^d = x^{d+1}$.
\end{fact}
Following our congestion game modeling, \Cref{thm:poa_bounds_original} provides a first (and crude) inefficiency guarantee which shows that the driving game becomes less and less efficient the higher is the maximum polynomial degree $d$. As discussed in~\Cref{sec:dg_as_cg}, degree $d$ represents the agents' sensitiveness to nearby cars, a factor which -- intuitively -- can lead to suboptimal equilibria. We expect however, that in practical scenarios $d$ should be rather small (current proximity costs such as~\eqref{eq:clear_cost_ds} assume $d=1$) since the the actual number of neighbouring cars should have a limited impact on the drivers' risk.

\subsection{Refining the Inefficiency Bounds for Driving Games}
While~\Cref{thm:poa_bounds_original} provides a range of PoA guarantees depending on the degree $d$ of the driving game, we will see that these are quite conservative bounds since they depend only on the agents' proximity cost $J^\text{cg}_i$ and neglect the personal agents' preferences. Indeed, we intuitively expect that in the limit where the proximity costs $J^\text{cg}_i$ become negligible compared to the personal ones $J^\text{per}_i$, the agents' costs become more and more ``decoupled" (since the 
$J_i^\text{per}$ only depends on $\gamma_i$) and thus the game \poa{} should tend to $1$. This is not captured by the guarantees obtained so far and serves as main motivation for the results presented next. 

To exploit the relative importance between personal and proximity costs, we define the following main quantity.
\begin{definition} Let $\strategyset_\textrm{NE}$ be the set of all NE and $\strategyset^\star = \arg \min_{\strategy \in \strategyset} C(\strategy)$ be the set of social optima strategies. We define $\alphas \in \mathbb{R}_+$ to be the largest constant such that for all agents $i$ and strategies $\strategy \in \strategyset_\textrm{NE}\cup \strategyset^\star$,
\begin{equation}
    \costper_i(\strategy_i) \geq \alphas  \cdot \costcg_i(\cgstrategy).
    \label{eq:alpha_condition}
\end{equation}
\label{alpha_star_definition}
\end{definition}
\vspace{-1.8em}
Note that $\alphas \geq 0$ since we have assumed positive costs without loss of generality. However, it is also not hard to imagine situations where $\alphas >0$. For instance, this is the case when $J^\text{per}_i$ is lower bounded by $\underline{J}^\text{per}$ (e.g., capturing total acceleration or fuel consumption) and $J^\text{cg}_i$ is upper bounded by $\overline{J}^\text{cg}$. In such a case, $\alphas \geq \underline{J}^\text{per} / \overline{J}^\text{cg}  > 0 $. Moreover, note that condition~\eqref{alpha_star_definition} needs to hold only for equilibria and socially optima policies, which is reasonable to assume that are not colliding and thus $\overline{J}^\text{cg}$ would be small.
In general, $\alphas$ is a (conservative) measure of the relative importance between personal and proximity costs. In the next theorem, we show that $\alphas$ can be used to obtain and characterize refined PoA guarantees.\looseness=-1
\begin{figure}[t]
	\centering
	\includegraphics[width=0.4\textwidth]{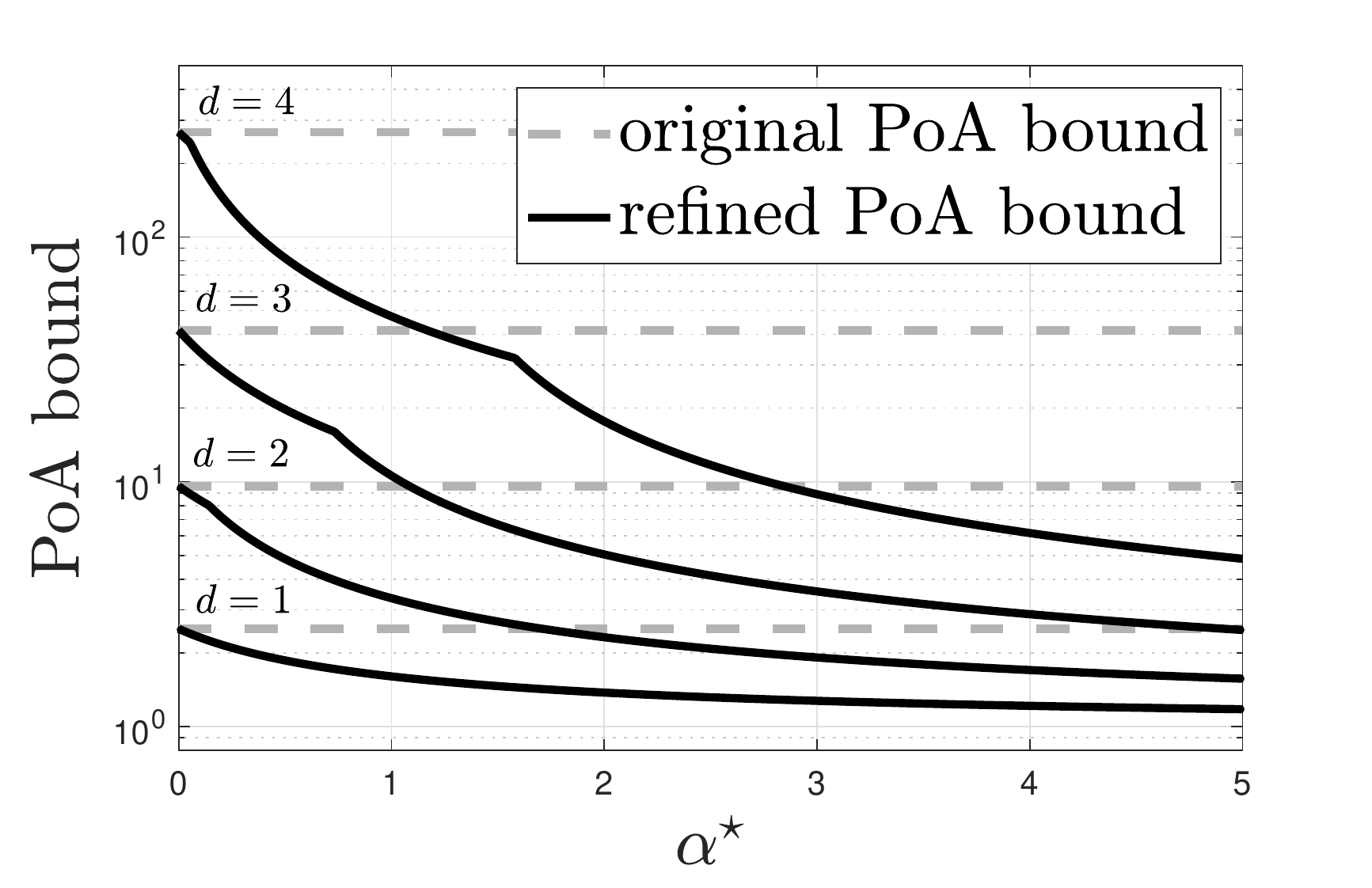}
	\caption{\poa{} upper bounds resulting from the congestion game formulation for different polynomial degrees $d$ (see ~\cref{thm:poa_bounds_original}), and refined based on the personal cost parameter $\alphas$ according to ~\cref{thm:poa_bounds_improved}.}
	\label{fig:improved_bounds}
\end{figure}
\begin{theorem}
\label{thm:poa_bounds_improved}
Consider driving games modeled as congestion games according to Section~\ref{sec:dg_as_cg}, with polynomial resource load function with non-negative coefficients and degree at most $d\in\naturals$. Moreover, consider $\alphas$ as per~\Cref{alpha_star_definition}. Then, \poa{} is upper bounded by
\begin{equation}
     \frac{(k+1)^{2d+1}-k^{d+1}(k+2)^d + \alphas\left((k+1)^{d+1}-k^{d+1}\right)}{(1+\alphas)\left((k+1)^{d+1}-k^{d+1}\right) - (k+2)^d+(k+1)^d} 
    \label{eq:cg_poa_upper_bound_improved}.
\end{equation}
Where $k\defeq \floor{\Psi_{d,\alphas}}$ and $\Psi_{d,\alphas}$ is the positive real solution to $x^{d+1}+\alphas x^{d+1} = (x+1)^d +\alphas$.
\end{theorem}
It can be verified that the bound~\eqref{eq:cg_poa_upper_bound_improved} above is strictly smaller than the one of~\Cref{thm:poa_bounds_original} for all $\alphas >0$ and, as expected, tends to $1$ as $\alphas \rightarrow \infty$. We visualize such refined guarantees 
in~\cref{fig:improved_bounds} for different degrees $d$ and as a function of $\alphas$. Notice that when $\alphas=0$ (in which case agents' personal costs are negligible) we retrieve the original upper bounds from~\Cref{thm:poa_bounds_original}. We outline the main steps to prove~\Cref{thm:poa_bounds_improved} below, while its full proof can be found in the Appendix for completeness.

\subsection{Proof Outline for~\Cref{thm:poa_bounds_improved}}
The overall proof follows the same methodology as the proof of~\cite[Thm. 7]{Aland2011Exact} but the intermediate steps have to be carefully adjusted to: 1) include agents' personal costs and 2) exploit the lower bounding constant $\alphas$ of~\cref{alpha_star_definition}. In what follows, we leave out the ``cg''-superscript in the strategies $\strategy$ in favor of a simpler notation. Moreover, we let $\mathcal{P}_d$ be the set of polynomials up to degree $d$.\looseness=-1

The obtained PoA guarantees utilize the fact that congestion games are $(\lambda, \mu)$-\emph{smooth} (in the sense of~\cite{Roughgarden2015}), i.e., there exist $\lambda > 0, \mu < 1$ such that for every pair $\strategy',\strategy$ of outcomes,
\begin{equation}
    \sum_{i\in \players} \cost_i(\strategy_i',\strategy_{-i}) \leq \lambda\socialcost(\strategy') +\mu\socialcost(\strategy).
    \label{eq:proof:smooth}
\end{equation}
According to \cite{Roughgarden2015}, this directly implies that
their \poa{} is upper-bounded by $\frac{\lambda}{1-\mu}$.
Moreover, as also noted in \cite{Roughgarden2015}, since we consider PoA of pure NE, it is sufficient that smoothness~\eqref{eq:proof:smooth} holds only for $\strategy' \in \strategyset_\textrm{NE}$ and $\strategy \in \Gamma^\star$. We will make use of such weaker condition to exploit the factor $\alphas$ of~\Cref{alpha_star_definition}.

We can use the above facts to prove the following lemma.
\begin{lemma}
Let $\alphas$ defined in~\Cref{alpha_star_definition}. Then, the PoA of our congestion driving game formulation is upper bounded by:
\begin{equation}
\begin{aligned}
\min_{\substack{\lati \in \mathbb{R} \\ \muti \in (0,1 +\alphas) }} \quad & \frac{\lati +\alphas }{1-\muti +\alphas }\\
\textrm{s.t.} \quad \quad \, &  y \cdot l(x+1) \leq \lati \cdot y \cdot \cost(y)+\muti \cdot x \cdot l(x)\\
  &\forall x,y \in \mathbb{N}_0, \cost\in \mathcal{P}_d  . \\
\end{aligned}
\label{eq:proof:min_alpha_eqvar}
\end{equation}
\begin{proof}
According to~\cite{Roughgarden2015}, a minimum PoA upper-bound can be achieve by minimizing $\frac{\lambda}{1-\mu}$ subject to the smoothness condition~\eqref{eq:proof:smooth} (for $\strategy' \in \strategyset_\textrm{NE}$ and $\strategy \in \Gamma^\star$) which, by plugging the driving games' costs becomes:
\begin{equation}
\begin{aligned}
    \sum_{i}&\costcg_i(\strategy_i',\strategyother) + \sum_{i}\costper_i(\strategy_i')  \leq \lambda \cdot \Big( \sum_{i}\costcg_i(\strategy') \\ &  +\sum_{i}\costper_i(\strategy_i')\Big) + \mu \cdot \Big(\sum_{i}\costcg_i(\strategy)+ \sum_{i}\costper_i(\strategy_i)\Big).
    \label{eq:proof:smooth_cg_w_per_cost}
\end{aligned}
\end{equation}

Let us now assume that $\lambda>1$ (we will prove in Fact~\ref{fact:lambdagreaterthan1} that this is without loss of generality). Then, 
by using~\cref{alpha_star_definition}, condition~\eqref{eq:proof:smooth_cg_w_per_cost} is satisfied whenever
\begin{equation}
\begin{aligned}
    \sum_{i}\costcg_i(\strategy_i',\strategyother)  \leq &  \underbrace{((1+\alphas)\lambda -\alphas)}_{\defeq \lati}\cdot \sum_{i}\costcg_i(\strategy') \\&+ \underbrace{(1+\alphas)}_{\defeq \muti}\mu \cdot \sum_{i}\costcg_i(\strategy),
    \label{eq:proof:alpha_smooth}
\end{aligned}
\end{equation}
where we have defined auxiliary smoothness constants $\lati \in \mathbb{R}$ and $\muti \in (0, 1 + \alphas)$.
Moreover, since $\sum_{i}\costcg_i(\strategy) = \sum_{i}\sum_{\resource\in\resourceset} \load_\resource^i(\strategy_i) \cdot \rescost_\resource(\load_\resource(\strategy)) = \sum_{\resource\in\resourceset} \load_\resource(\strategy) \cdot \rescost_\resource(\load_\resource(\strategy))$, and a deviation by a single player means that the load on resource $\resource$ increases at most by $1$, a sufficient condition for~\eqref{eq:proof:alpha_smooth} to hold is:
\begin{equation*}
\begin{aligned}
\sum_{\resource\in\resourceset} \load_\resource(\strategy') \cdot \rescost_\resource(\load_\resource(\strategy) + 1) &  \leq 
\sum_{\resource\in\resourceset}\Big[\lati \cdot \load_\resource(\strategy') \cdot \rescost_\resource(\load_\resource(\strategy') \\&  \qquad + \muti \cdot \load_\resource(\strategy) \cdot \rescost_\resource(\load_\resource(\strategy))\Big].
\end{aligned}
\end{equation*}
Since $\frac{\lambda}{1- \mu} = \frac{\lati +\alphas }{1-\muti +\alphas } $ and $\load_\resource(\strategy')$ and $\load_\resource(\strategy)$ are both in $\mathbb{N}_0$, finding a \poa{} upper bound for our driving game can be formulated as finding a solution to~\eqref{eq:proof:min_alpha_eqvar}.
\end{proof}
\end{lemma}

The remaining part of the proof utilizes a series of intermediate lemmas which can be obtained adapting the ones from~\cite{Aland2011Exact} to our modified problem~\eqref{eq:proof:min_alpha_eqvar}; we refer to the Appendix for their full claims and proofs. Essentially:
\begin{itemize}
    \item[1)] Problem~\eqref{eq:proof:min_alpha_eqvar} has the same solution as \begin{equation}
        \inf_{\muti \in (0, 1+\alphas)}{\left\{\max_{x\in\mathbb{N}_0}{\left\{\frac{\lati^\star + \alphas}{1-\muti+\alphas}\right\}}\right\}}, \label{eq:final _obj}
    \end{equation}
    with $\lati^\star = (x+1)^d-\muti\cdot x^{d+1}$. 
\item[2)] The optimal values of $\muti$ and $x$ that solve~\eqref{eq:final _obj} are $\mutis = \frac{(k+2)^d - (k+1)^d}{(k+1)^{d+1} -k^{d+1}}$ and $k= \floor{\Psi_{d,\alphas}}$, respectively, where  $\Psi_{d,\alphas}$ is the positive real solution to $x^{d+1}+\alphas x^{d+1} = (x+1)^d +\alphas$.
\end{itemize}
Finally,~\Cref{thm:poa_bounds_improved} is proven by plugging $\lati^\star$ and $\mutis$ into the objective of~\eqref{eq:final _obj}. We are left, however, with showing that considering $\lambda \geq 1$ is without loss of generality. This is equivalent to showing that $\lambda^\star = \frac{\lati^\star + \alphas}{1 +\alphas}$ satisfies $\lambda^\star \geq 1$ and thus proving the following fact. 
\begin{fact}
The value $\lati^\star$ of $\lati$ that solves problem~\eqref{eq:proof:min_alpha_eqvar} satisfies $\lati^\star \geq 1$. 
\end{fact}
\begin{proof}
   Combining the expressions of $\lati^\star$ and $\muti^\star$ from points 1) and 2) above, it holds $\lati^\star \geq 1$ whenever: 
   \begin{align*}
      & \frac{(k+1)^d - 1}{k^{d+1}}  \geq \frac{(k+2)^d - (k+1)^d}{(k+1)^{d+1} -k^{d+1}}, \quad  \forall d \\
      & \substack{\forall d > 0 \\ \Longleftarrow} \:  \frac{(k+1)^{d+1}}{k^{d+1}}  \geq 1 +  \frac{(k+2)^d - (k+1)^d}{(k+1)^{d} -1} = \frac{(k+2)^d - 1}{(k+1)^{d} -1}  \\
     &  \substack{\forall d > 0 \\ \Longleftrightarrow} \quad \frac{(k+1)^{d+1}}{(k+2)^d - 1} \geq  \frac{k^{d+1}}{(k+1)^{d} -1},
   \end{align*}
   which is satisfied since $f(x) = x^{d+1}/[(x+1)^d -1]$ is monotone for all $x>0, d>0$.
\end{proof}
\section{Experiments}
In this section, we present an experimental case study to empirically assess the possible efficiency gap of various driving scenarios. 
Before presenting our results, we recall that the analytical \poa{} bounds derived in~\Cref{sec:ineff_bounds} hold for open-loop strategies, where the players commit at the beginning of the game to the entire trajectory. 
Due to uncertainty about others and computational limits, a more realistic setup is to consider \emph{feedback} strategies (i.e., policies) for the agents.
While an analytic bound of the \poa{} in the feedback case is more involved and it is for now delegated to future works, in our experiments we consider stochastic feedback policies and provide empirical evidence suggesting that equilibrium policies have a social cost comparable to the socially optimum one.\looseness=-1

\paragraph*{Approximating the \poa{}}
We take instances of driving games and learn feedback policies for the agents using~\gls{marl}~\cite{Zhou2020SMARTS}. 
Since the goal is to observe the \poa{}, we consider both centralized training paradigms ($\sim$social optimum) as well as decentralized ones (self-interested agents, $\sim$equilibria). In the first case, 
we train a joint policy which receives as input the stacked observations of all the agents and outputs the action commands for all of them. To compute other \gls{ne} policies instead, we train individual policies for each agent mapping its observations to a corresponding control input. We can consider the learned policies to represent an equilibrium thanks to the recent results for potential games and gradient-based learning methods, e.g. by~\cite{Mao2022OnLearning,Ding2022IndependentConvergence}.
To approximate the value of the worst equilibrium required by the \poa{}, we run many decentralized trainings to find equilibrium policies that are qualitatively different in the homotopic sense and select the one that performs the worst. In~\Cref{tab:experimental_poas} we indicate with ``\# of comput. policies'' the number of training runs that converged to stable policies (either centralized or decentralized). We highlight that the centralized training paradigm requires significantly more computation, since the joint policy action space grows exponentially with the number of agents. Moreover, it may be infeasible in real-life scenarios since it requires coordination and communication among the agents. Nevertheless, we consider it here as an ideal performance benchmark to reach socially optimal outcomes and thus obtain approximate PoAs.

\subsection{Experimental Setup}
All the experiments were conducted in the \gls{smarts} framework~\cite{Zhou2020SMARTS}. 
We benchmark two scenarios (intersection and merging) with a varying number of players ($2-4$).
\paragraph*{Observation and Action Spaces}
The observation space for each agent is a stacked vector including its own state and the one of all neighboring vehicles within a 50 meter radius.
More specifically, each state includes the position relative to the goal, the distance to the center of the current lane, the speed, the steering angle, and a list of heading errors. Moreover, it also includes the stacked states of the two most recent time steps.

The action space of each agent consists of four high level actions at each time step, namely $\{$\texttt{Keep Lane}, \texttt{Slow Down}, \texttt{Change Lane Left}, \texttt{Change Lane Right}$\}$. The low-level control is handled by the lane following controller implemented in \gls{smarts}~\cite{Zhou2020SMARTS}.

\paragraph*{Cost/Reward Function}
The rewards of the individual agents are comprised of proximity costs and personal objectives. 
The proximity costs penalize, for each time step, the agents' distance to nearby cars and are computed as in~\eqref{eq:clear_cost_ds}.
Moreover, to ensure that the agents remain on the road and that they do not crash, we associate one-time negative rewards to these events. 
The personal objectives consist of a constant cost for each time step the agent takes to complete the mission which encourages faster completion, together with a comfort cost penalizing quadratically high accelerations. 
Moreover, agents receive a one-time positive reward for reaching the defined destination. Finally, we experienced that adding a small reward at each time step when the agent decreases the overall distance to the goal, fostered faster convergence of the learned policies. 
\paragraph*{Training}
The training of the agents' policies is carried out using Proximal Policy Optimization (PPO) algorithm both  centralized and decentralized, with Adam optimizer.

\begin{figure}[t]
\centering
\begin{subfigure}{.5\linewidth}
  \centering
  \includegraphics[scale=.17]{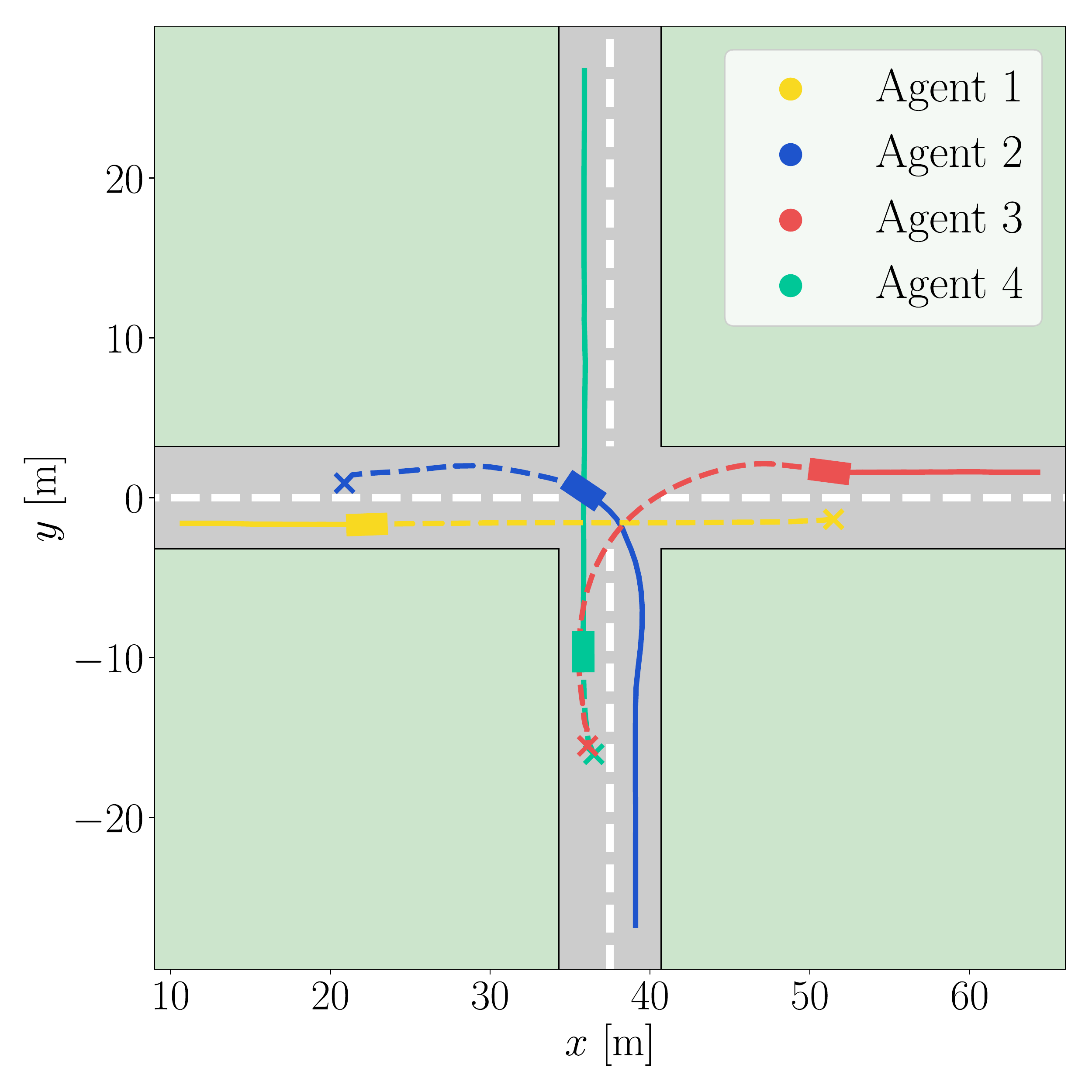}
\end{subfigure}%
\begin{subfigure}{.5\linewidth}
  \centering
  \includegraphics[scale=.17]{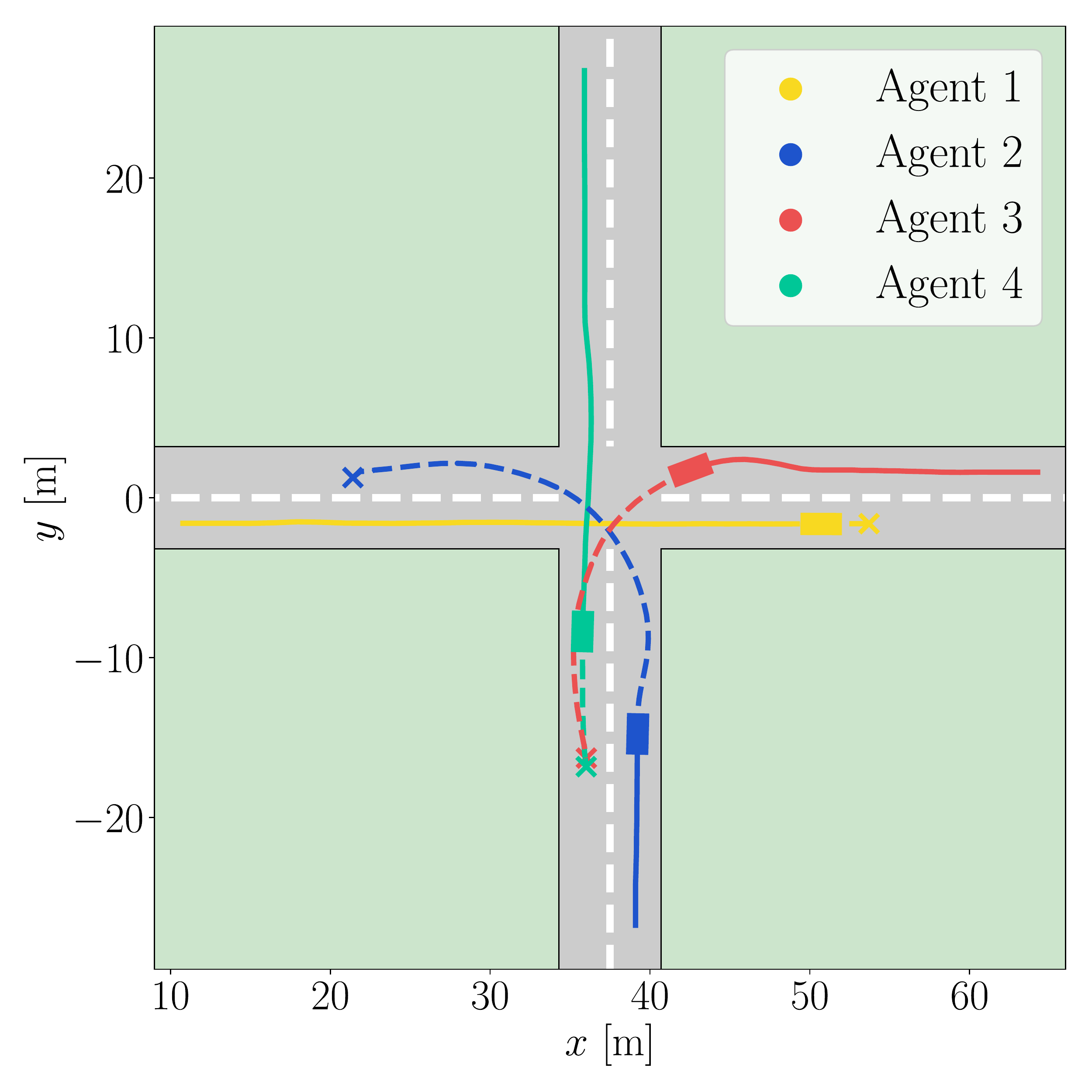}
\end{subfigure}
\vspace{0.1em}
\caption{
Different equilibrium policies learned in the $4$-player case. 
Each plot shows the average resulting behavior from trained policies using the decentralized paradigm.
As the policies are stochastic, the averaging has been done over $200$ game realizations.
The past time steps are drawn with a solid line while future time steps are visualized by the dashed lines.
One can observe that some learned equilibrium policies are more efficient than others.
In the left plot, the learned equilibrium is close to the minimum of the social cost, whereas other training runs converge to more inefficient equilibria (right plot). One can clearly observe how the different learned driving ``culture'' result in a different order in which the vehicles cross the intersection.
Similar results were obtained also in the merging scenario of~\Cref{fig:driving_game}.
}
\label{fig:intersection}
\end{figure}

\begin{table}[t]
	\centering
	\begin{threeparttable}
		\setlength\tabcolsep{0pt} %
		\captionsetup{labelsep=period, skip=5pt}
		\caption{Summary Results}
    	\label{tab:experimental_poas}
		\begin{tabular*}{\columnwidth}{l@{\extracolsep{\fill}}*{7}{c}}
			\toprule
			Scenario  & \multicolumn{3}{c}{\textbf{Intersection}}  & \phantom{a}& \multicolumn{3}{c}{\textbf{Merging}}  \\
			\cmidrule{2-4} \cmidrule{6-8}
		      \# of players	& 2&3&4 && 2&3&4 \\ \midrule
                \# of comput. policies & 8& 6& 5&& 10& 10& 7\\
			\textbf{Observed \poa{}} & 1.28 & 1.22 & 1.18 && 1.41 & 1.16 & 1.27  \\
			\bottomrule
		\end{tabular*}
	\end{threeparttable}
\end{table}

\subsection{Results and Discussion}
For each driving scenario and training run, we inspect convergence (in terms of social cost) and take the corresponding agents' policies to represent an equilibrium (in case of decentralized training) or a social local optimum (when using centralized training). PoA is computed using the worst-performing equilibrium policy and the socially optimal one observed. Because policies are stochastic, PoA is computed as the average over 200 evaluations of such policies.

Interestingly, we observe that from multiple decentralized training runs we obtain different ``driving cultures'' for the agents. Two concrete examples are shown in~\Cref{fig:intersection} for the intersection scenario, where different training runs delivered different orders in which the vehicles learn to give each other the right of way. For the merging scenarion, two different policies are visualized in~\Cref{fig:driving_game}. 
These qualitatively different equilibria can be formally described at the topological level with braid's theory and homotopic classes (see for an overview the topological models section in~\cite{Wang2022SocialPerspectives}).
More quantitatively, in~\Cref{tab:experimental_poas}
we report the number of observed NEs for the considered scenarios and number of players, toghether with their corresponding observed PoAs.
Even though the observed PoAs underestimate of the real ones (which are not feasible to compute as they require computing all possible NEs), all observed equilibria are quite efficient according to our observed outcomes and display a PoA $<$ 1.5. 
This suggests that agents in a driving game can reach very efficient outcomes (i.e., NE policies) in a decentralized fashion via independent learning, without employing centralized and/or complex communication protocols. 
In other words, the trained centralized policy leads to marginal improvements in terms of social cost, albeit requiring a significantly higher computational complexity (scales exponentially with the number of agents) as well as agents' coordination and communication. We further notice that the observed PoAs are lower than their bounds of~\Cref{sec:dg_as_cg}.
Besides the mismatch between theory and experiments, this is also expected since PoA guarantees are indeed robust. They apply to any game in such a class, and to any equilibrium and therefore can be overly conservative in practice (see, e.g. \cite{Roughgarden2015, Sessa2019b}).

\section{Conclusions and Outlook}
We have considered the problem of bounding the inefficiency of equilibria in driving games. To this end, we showed that such games can be formulated as a particular type of congestion games and that this allows obtaining rigorous novel PoA bounds as a function of game-dependent parameters. Finally, we considered various driving scenarios and reported empirical evidence on the efficiency of equilibrium policies computed via decentralized \gls{marl}.

The obtained PoA bounds are the first of their kind in the robotics literature and they open-up interesting related research questions.   
First, there is still quite a gap between our theoretical and experimental setup, yet the reported evidence suggests that PoA bounds could perhaps be derived for such more complex case. Second, in line with our theoretical bounds of \Cref{sec:ineff_bounds}, it would be interesting to observe how the empirical PoAs change as a function of game-dependent parameters, albeit this requires significant computational resources. Finally, it would be meaningful to study whether the obtained PoA bounds are tight; this has been shown for general congestion games~\cite{Aland2011Exact}, but it is obvious if this applies to our specific driving setup too.

\bibliographystyle{IEEEtran}
\bibliography{references_mendeley}
\appendix
\label{appendix_proof}

In this section, we prove the refined PoA guarantees  of~\cref{thm:poa_bounds_improved}. The overall proof follows the same methodology as the proof of~\cite[Thm. 7]{Aland2011Exact} but the intermediate steps have to be carefully adjusted to: 1) include agents' personal costs and 2) exploit the lower bounding constant $\alphas$ of~\cref{alpha_star_definition}. In what follows, we leave out the ``cg''-superscript in the strategies $\strategy$ in favor of a simpler notation. Moreover, we let $\mathcal{P}_d$ be the set of polynomials up to degree $d$.\looseness=-1

\subsection{Game smoothness}
The obtained guarantees utilize the following notion of \emph{smoothness} from~\cite{Roughgarden2015}.
\begin{definition}[$(\lambda, \mu)$-smooth game~\cite{Roughgarden2015}]
A game is $(\lambda, \mu)$-\emph{smooth} $(\lambda > 0, \mu < 1)$ if for every pair $\strategy',\strategy$ of outcomes,
\begin{equation}
    \sum_{i\in \players} \cost_i(\strategy_i',\strategy_{-i}) \leq \lambda\socialcost(\strategy') +\mu\socialcost(\strategy).
    \label{eq:smooth}
\end{equation}
Smoothness is a widely adopted condition to prove PoA bounds, according to the following theorem.
\end{definition}
\begin{fact}[Section 2.1 of \cite{Roughgarden2015}]
\label{thm:smooth_PoA_bound}
    If a game is $(\lambda, \mu)$-\emph{smooth}, then the \poa{} is upper-bounded by $\frac{\lambda}{1-\mu}$.
\end{fact}
\begin{remark}\label{remark:weak_smoothness}
As noted by \cite{Roughgarden2015}, to bound the PoA of pure Nash equilibria (which is the notion of PoA that we consider in this work -- see~\Cref{def:poa}), the bound of \Cref{thm:smooth_PoA_bound} holds even when the smoothness condition~\eqref{eq:smooth} is only satisfied for $\strategy' \in \strategyset_\textrm{NE}$ and $\strategy = \arg\min_{\gamma\in \strategyset} \socialcost(\strategy)$ (hence, not for all strategy pairs). We will make use of such weaker condition to exploit the factor $\alphas$ of~\Cref{alpha_star_definition}.
\end{remark}

\subsection{Auxiliary lemmas}
We can use the above results to prove the following intermediate lemma.
\begin{lemma}
Consider constant $\alphas$ defined in~\Cref{alpha_star_definition}. Then, the PoA of our congestion driving game (with added personal cost) is upper bounded by:
\begin{equation}
\begin{aligned}
\min_{(\lambda,\mu) \in \mathbb{R}_+\times (0,1)} \quad & \frac{\lambda}{1-\mu}\\
\textrm{s.t.} \quad \quad \quad \, &  y \cdot \cost(x+1) \leq ((1+\alphas)\lambda-\alphas) \cdot y \cdot \cost(y)\\
&\quad \quad \qquad \qquad \! +(1+\alphas)\mu \cdot x \cdot \rescost(x)\\
  &\alphas \in [0, \infty), \forall x,y \in \mathbb{N}_0, \cost\in \mathcal{P}_d  \\
\end{aligned}
\label{eq:min_alpha}
\end{equation}
\begin{proof}
We begin by plugging the player costs into the smoothness condition~(\ref{eq:smooth}) which, in light of~\Cref{remark:weak_smoothness}, is required to hold only for $\strategy' \in \strategyset_\textrm{NE}$ and $\strategy \in \arg\min_{\gamma\in \strategyset} \socialcost(\strategy)$:
\begin{equation}
\begin{aligned}
    \sum_{i}\costcg_i&(\strategy_i',\strategyother) + \sum_{i}\costper_i(\strategy_i')  \leq \\ &\lambda \cdot \left( \sum_{i}\costcg_i(\strategy') +\sum_{i}\costper_i(\strategy_i')\right) \\&+ \mu \cdot \left(\sum_{i}\costcg_i(\strategy)+ \sum_{i}\costper_i(\strategy_i)\right).
    \label{eq:smooth_cg_w_per_cost}
\end{aligned}
\end{equation}

Let us now assume that $\lambda>1$ (we will prove in Fact~\ref{fact:lambdagreaterthan1} that this is without loss of generality). Then, 
by using~\cref{alpha_star_definition}, condition~\eqref{eq:smooth_cg_w_per_cost} is satisfied whenever

\begin{equation}
\begin{aligned}
    \sum_{i}\costcg_i(\strategy_i',\strategyother)  \leq &  ((1+\alphas)\lambda -\alphas)\cdot \sum_{i}\costcg_i(\strategy') \\&+ (1+\alphas)\mu \cdot \sum_{i}\costcg_i(\strategy).
    \label{eq:alpha_smooth}
\end{aligned}
\end{equation}

Moreover, since $\sum_{i}\costcg_i(\strategy) = \sum_{i}\sum_{\resource\in\resourceset} \load_\resource^i(\strategy_i) \cdot \rescost_\resource(\load_\resource(\strategy)) = \sum_{\resource\in\resourceset} \load_\resource(\strategy) \cdot \rescost_\resource(\load_\resource(\strategy))$,
the left hand side of Ineq.~(\ref{eq:alpha_smooth}) can be reformulated into
\begin{equation}
\begin{aligned}
    &\sum_{i}\costcg_i(\strategy_i',\strategyother) =\\  \sum_i\sum_{\resource\in\resourceset} \load_\resource^i(&\strategy_i') \cdot \rescost_\resource(\load_\resource(\strategy_i, \strategyother) - \load_\resource^i(\strategy_i) + \load_\resource^i(\strategy_i'))
    \label{eq:cg_total_cost}
\end{aligned}
\end{equation}
with $\load_\resource(\strategy_i, \strategyother) - \load_\resource^i(\strategy_i) + \load_\resource^i(\strategy_i') = \load_\resource(\strategy_i', \strategyother)$ being the total load on resource $\resource$ given that player $i$ unilaterally deviated from $\strategy_i$ to $\strategy_i'$. Since a deviation by a single player could at maximum mean that the load on resource $\resource$ increases by one, we can formulate an upper bound on~\eqref{eq:cg_total_cost}:
\begin{equation}
    \sum_{i}\costcg_i(\strategy_i',\strategyother) \leq \sum_{\resource\in\resourceset} \load_\resource(\strategy') \cdot \rescost_\resource(\load_\resource(\strategy) + 1).
    \label{eq:upper_bound_cg_cost}
\end{equation}
This leads us to an other sufficient condition of smoothness by combining~\eqref{eq:upper_bound_cg_cost} and~\eqref{eq:alpha_smooth}:

\begin{equation}
\begin{aligned}
\sum_{\resource\in\resourceset} &\load_\resource(\strategy') \cdot \rescost_\resource(\load_\resource(\strategy) + 1) \leq \\ &
\sum_{\resource\in\resourceset}\Big[((1+\alphas)\lambda -\alphas)\cdot \load_\resource(\strategy') \cdot \rescost_\resource(\load_\resource(\strategy')) \\& \qquad + (1+\alphas)\mu \cdot \load_\resource(\strategy) \cdot \rescost_\resource(\load_\resource(\strategy))\Big].
\end{aligned}
\end{equation}

Thus, due to~\Cref{thm:smooth_PoA_bound} and since $\load_\resource(\strategy')$ and $\load_\resource(\strategy)$ are both in $\mathbb{N}_0$, finding a \poa{} upper bound for our driving game can be formulated as finding a solution to the minimization problem~\eqref{eq:min_alpha}.
\end{proof}
\end{lemma}

 With the intent of obtaining an explicit PoA bound, we note that the minimization problem~\eqref{eq:min_alpha} can be brought to the following equivalent form with $\lati \defeq ((1+\alphas)\lambda-\alphas)$ and $\muti \defeq (1+\alphas)\mu$:
\begin{equation}
\begin{aligned}
\min_{\substack{\lati \in \mathbb{R} \\ \muti \in (0,1 +\alphas) }} \quad & \frac{\lati +\alphas }{1-\muti +\alphas }\\
\textrm{s.t.} \quad \quad \, &  y \cdot l(x+1) \leq \lati \cdot y \cdot \cost(y)+\muti \cdot x \cdot l(x)\\
  &\forall x,y \in \mathbb{N}_0, \cost\in \mathcal{P}_d  . \\
\end{aligned}
\label{eq:min_alpha_eqvar}
\end{equation}
 Moreover, the following three main Lemmas can be obtained similarly to~\cite{Aland2011Exact}. We will use the notation $[n]_0\defeq\{0,1,\hdots, n\}$.

 \begin{lemma}[Adaptation of Lemma 5.1 from~\cite{Aland2011Exact}]
    Let $\muti \in (0,\infty)$ and $x \in \mathbb{R}_{\geq 0}$. Define $g:\mathbb{R}_{\geq 0} \rightarrow \mathbb{R}$, $g(x) \defeq (x+1)^r-\muti\cdot x^{r+1}$, then it holds for all $d,r \in \mathbb{R}_{\geq 0}$ with $d > r$ and $g(r) \geq 0$ that $g(d) \geq g(r)$.
    \label{lemma:51}
\end{lemma}
 %

\iffalse
\nk{I excluded the following Lemma, because it would play a role in the proof of~\cref{claim2}, but we also omitted the proof to this claim.}

 %
\begin{lemma}[adaptation of Lemma 5.2 from~\cite{Aland2011Exact}]
    Let $\muti \in (0,\infty)$, $d \in \mathbb{N}$ and $\alphas \in [0, \infty)$. Define $g:\mathbb{R}_{\geq 0} \rightarrow \mathbb{R}$, $g(r) \defeq (x+1)^r-\muti\cdot x^{r+1} + \alphas$, then it holds that $g$ has exactly one local maximum at some $\xi \in \mathbb{R}_{\geq 0}$. Moreover, $g$ is strictly increasing in $[0, \xi)$ and strictly decreasing in $(\xi, \infty)$.
\end{lemma}
 %
 
 We omit the proofs for~\cref{lemma:51} and \cref{lemma:52} since the proof of \cite{[Lemma 5.2]Aland2011Exact} also works for our slight adaptation, i.e. additionally considering $\muti > 1$.
 %
\fi

 We omit the proof for~\cref{lemma:51} since the proof of \cite[Lemma 5.1]{Aland2011Exact} also works for our slight adaptation, i.e. additionally considering $\muti > 1$.

\begin{lemma}[Adaptation of Lemma 5.7 from~\cite{Aland2011Exact}]
    Let $d \in \mathbb{N}_0$. Then it holds for all $\muti \in (0,\infty)$ that 
    \begin{equation}
    \begin{aligned}
        &\max_{x\in \mathbb{N}_0, y \in \mathbb{N}}{\left\{\left(\frac{x+1}{y}\right)^d-\muti\cdot\left(\frac{x}{y}\right)^{d+1}\right\}} \\&= \max_{x\in \mathbb{N}_0}{\left\{(x+1)^d-\muti\cdot x^{d+1}\right\}}.
    \end{aligned}
    \label{eq:57}
    \end{equation}
    \begin{proof}
    By~\cite[Lemma 5.2]{Aland2011Exact} the maximum of the right hand side of~\eqref{eq:57} exists and is unique.
    We define $g: \mathbb{N}_0 \times \mathbb{N}\times [0,\infty) \rightarrow \mathbb{R}$ as
    \begin{equation*}
    g(x,y,\muti)\defeq \left(\frac{x+1}{y}\right)^d -\muti\cdot\left(\frac{x}{y}\right)^{d+1}. \end{equation*}
    Now we show that, $\forall x\in\mathbb{N}_0, y\in\mathbb{N}, \exists \hat{x}\in \mathbb{N}_0$ such that $\forall \muti\in[0,\infty)$,
    \begin{equation*}
        g(\hat{x}, 1, \muti) \geq g(x,y,\muti).
    \end{equation*}
    If $y\geq x+1$, we have $\forall\muti\in[0,\infty)$ that $g(0,1,\muti) = 1\geq g(x,y,\muti)$, meaning that $0$ is an appropriate choice for $\hat{x}$. Therefore, we can only consider the case where $y\leq x$.\\
    Define
    \begin{equation*}
        \hat{x} \defeq \ceil[\bigg]{\frac{x+1-y}{y}}
    \end{equation*}
     Because of $x\geq y$, $x$ can be written as $b_1\cdot y+b_2$ for some $b_1\in\mathbb{N}$ and $b_2\in [y-1]_0$. This shows that
    \begin{equation*}
        \hat{x} = b_1 -1 + \ceil[\bigg]{\frac{b_2+1}{y}} = b_1 = \floor[\bigg]{\frac{x}{y}}
    \end{equation*}
    Now it holds $\forall \muti \in [0,\infty)$ that 
    \begin{equation*}
    \begin{aligned}
        g(\hat{x}, 1, \muti) &= \left(\ceil[\bigg]{\frac{x+1-y}{y}} +1\right)^d - \muti\cdot \floor[\bigg]{\frac{x}{y}}^{d+1} \\&=  \ceil[\bigg]{\frac{x+1}{y}}^d - \muti\cdot \floor[\bigg]{\frac{x}{y}}^{d+1} \\
        &\geq \left(\frac{x+1}{y}\right)^d - \muti \cdot \left(\frac{x}{y}\right)^{d+1} = g(x,y,\muti)
    \end{aligned}
    \end{equation*}
    which proves the Lemma.
    \end{proof}
    \label{lemma:57}
\end{lemma}
\begin{lemma}[Adaptation of Lemma 5.8 from~\cite{Aland2011Exact}]
    The minimization problem~(\ref{eq:min_alpha_eqvar}) has the same solution as 
    \begin{equation}
        \inf_{\muti \in (0, 1+\alphas)}{\left\{\max_{x\in\mathbb{N}_0}{\left\{\frac{(x+1)^d-\muti\cdot x^{d+1} + \alphas}{1-\muti+\alphas}\right\}}\right\}}
        \label{eq:lemma_min_prob_reformulation}
    \end{equation}
    \begin{proof}
    We start by simplifying the constraints from~(\ref{eq:min_alpha_eqvar}) restated here.
    \begin{equation*}
        \forall x,y \in \mathbb{N}_0, \cost\in \mathcal{P}_d: y \cdot \cost(x+1) \leq \lati \cdot y \cdot \cost(y)+\muti \cdot x \cdot \cost(x)
    \end{equation*}
    Note that for $y=0$ the condition is trivially satisfied. Furthermore, since $l\in\mathcal{P}_d$, is a linear combination of monomials, the constraints can be reformulated to hold for all monomials of a up to a degree of $d$:
    \begin{equation*}
        \forall x \in \mathbb{N}_0, y \in \mathbb{N}, r \in [d]_0: y \cdot (x+1)^r \leq \lati \cdot y^{r+1}+\muti \cdot x^{r+1}.
    \end{equation*}
    Now, dividing by $y^{r+1}$ and rearranging the terms leads to
    \begin{equation*}
        \forall x \in \mathbb{N}_0, y \in \mathbb{N}, r \in [d]_0: \lati \geq \left(\frac{x+1}{y}\right)^r - \muti\cdot\left(\frac{x}{y}\right)^{r+1}.
    \end{equation*}
    Then, applying~\cref{lemma:57} we get
    \begin{equation*}
        \forall x \in \mathbb{N}_0, r \in [d]_0: \lati \geq (x+1)^r-\muti\cdot x^{r+1},
    \end{equation*}
   which, using~\cref{lemma:51}, results in
    \begin{equation}
        \forall x \in \mathbb{N}_0: \lati \geq (x+1)^d-\muti\cdot x^{d+1}.
        \label{eq:simplified_constraint}
    \end{equation}
    This simplification enables us to express $\lati$ in terms of $x$ and $\muti$ by using $\lati^\star = (x+1)^d-\muti\cdot x^{d+1}$. Finally, since the constraint~(\ref{eq:simplified_constraint}) has to hold for all $x\in\mathbb{N}_0$, the problem~(\ref{eq:min_alpha_eqvar}) can be reformulated as~(\ref{eq:lemma_min_prob_reformulation}).
    \end{proof}
    \label{lemma:58}
\end{lemma}

\subsection{Proof of~\Cref{thm:poa_bounds_improved}}

Using the previous lemmas, we can finally prove~\cref{thm:poa_bounds_improved}.
We directly use the equivalent reformulation~\eqref{eq:lemma_min_prob_reformulation} that was proven in~\cref{lemma:58}. Moreover, we define $g: (0,1+\alphas)\times \reals \rightarrow \reals$,
\begin{equation*}
    g(\muti, x) \defeq \frac{(x+1)^d-\muti\cdot x^{d+1}+\alphas}{1-\muti+\alphas}.
\end{equation*}
which is differentiable on $(0,1+\alphas)\times \reals$. To solve~\eqref{eq:lemma_min_prob_reformulation}, we first need to state and prove some claims. As a reminder, we let $k\defeq \floor{\Psi_{d,\alphas}}$ where $\Psi_{d,\alphas}$ is the positive real solution to $x^{d+1}+\alphas x^{d+1} = (x+1)^d +\alphas$.
\begin{claim}
    There exists a $\mutis \in (0, 1+\alphas)$ with $g(\mutis, k) = g(\mutis, k+1)$.
    \begin{proof}[Proof of claim]
        We first solve $g(\mutis, k) = g(\mutis, k+1)$ for $\mutis$.
        \begin{equation*}
        \begin{aligned}
            &(k+1)^d-\mutis\cdot k^{d+1} = (k+2)^d - \mutis\cdot (k+1)^{d+1} \\
            &\Longleftrightarrow \quad\mutis = \frac{(k+2)^d - (k+1)^d}{(k+1)^{d+1} -k^{d+1}}
        \end{aligned}
        \end{equation*}
        Since $(k+2)^d - (k+1)^d>0$ and $(k+1)^{d+1} -k^{d+1} > 0$, $\mutis$ is greater than zero. We also know that 
        \begin{equation*}
        \begin{aligned}
            (1+\alphas)\cdot k^{d+1} &< (k+1)^d + \alphas \\
            (1+\alphas)\cdot (k+1)^{d+1} &> (k+2)^d + \alphas
        \end{aligned}
        \end{equation*}
        Dividing by $(1+\alphas)$ and plugging the resulting upper-bound on $k^{d+1}$ and lower bound on $(k+1)^{d+1}$ into the fraction above yields
        \begin{equation*}
            \mutis < \frac{\left[(k+2)^d - (k+1)^d\right]\cdot (1+\alphas)}{(k+2)^d - (k+1)^d} = 1+\alphas
        \end{equation*}
        Therefore, it is shown that $\muti\in(0,1+\alphas)$.
    \end{proof}
    \label{claim1}
\end{claim}

\begin{claim}
    The $(\mutis, x^\star) \in (0,1+\alphas) \times \naturals_0$ that satisfy $g(\mutis, x^\star) = g(\mutis, x^\star+1)$ constitute the solution to $\max_{x\in\naturals_0}\{g(\mutis, x)\}$.
    \label{claim2}
\end{claim}

\begin{claim}
    For all $\muti\in(0,\mutis)$ it holds that $g(\muti, k+1) > g(\mutis, k)$ and for all $\muti\in (\mutis, 1+\alphas)$ it holds that $g(\muti, k) > g(\mutis, k)$.
    \label{claim3}
\end{claim}

We omit the proofs of \cref{claim2} and \cref{claim3}, since they are analogous to the proofs of Claims 2 and 3 from~\cite{Aland2011Exact}.

We have seen that the solution to~(\ref{eq:lemma_min_prob_reformulation}) is upper-bounded by $\max_{x\in\naturals_0}\{g(\mutis, x)\}$ by~\cref{claim1} and that $\max_{x\in\naturals_0}\{g(\mutis, x)\} = g(\mutis, k)$ by~\cref{claim2}. \cref{claim3} implies that ~(\ref{eq:lemma_min_prob_reformulation}) is also lower-bounded by $g(\mutis, k)$.

Now all that is left is to plug the found expression for $\mutis$ into $g(\mutis, k)$ which leads to expression~(\ref{eq:cg_poa_upper_bound_improved}) from~\cref{thm:poa_bounds_improved} and thereby proves~\cref{thm:poa_bounds_improved}.
\hfill \qedsymbol

\begin{fact}\label{fact:lambdagreaterthan1}
The value $\lambda^\star$ of $\lambda$ that solves problem~\eqref{eq:min_alpha} satisfies $\lambda^\star\geq 1$. 
\end{fact}
\begin{proof}
We show that $\lati^\star \geq 1$. The result then follows by having defined $\lambda^\star = \frac{\lati^\star + \alphas}{1 +\alphas}$.
   From proof of \Cref{lemma:58}, we know that $\lati^\star = (k+1)^d-\muti^\star \cdot k^{d+1}$. Hence, plugging the expression of $\muti^\star$, it holds $\lati^\star \geq 1$ whenever: 
   \begin{align*}
      &  \frac{(k+1)^d - 1}{k^{d+1}}  \geq \frac{(k+2)^d - (k+1)^d}{(k+1)^{d+1} -k^{d+1}}, \,  \forall d \\ &
       \substack{\forall d > 0 \\  \Longleftarrow} \:  \frac{(k+1)^{d+1}}{k^{d+1}}  \geq 1 +  \frac{(k+2)^d - (k+1)^d}{(k+1)^{d} -1} = \frac{(k+2)^d - 1}{(k+1)^{d} -1}  \\ &
       \substack{\forall d > 0 \\ \Longleftrightarrow} \quad \frac{(k+1)^{d+1}}{(k+2)^d - 1} \geq  \frac{k^{d+1}}{(k+1)^{d} -1},
   \end{align*}
   which is satisfied since $f(x) = x^{d+1}/[(x+1)^d -1]$ is monotone for all $x>0, d>0$. The latter is true since its derivative $f'(x) = \frac{x^d[(d+1)((x+1)^d-1)- dx(x+1)^{d-1}]}{[(x+1)^d-1]^2} \geq 0 \Longleftrightarrow (d+1)((x+1)^d-1)- dx(x+1)^{d-1} = (d+1)((x+1)(x+1)^{d-1} -1) - dx(x+1)^{d-1} = (d+1)(x(x+1)^{d-1} +(x+1)^{d-1} -1)- dx(x+1)^{d-1}= d(x+1)^{d-1} - d + x(x+1)^{d-1} +(x+1)^{d-1} -1 = (x+1)^{d-1}(d + x + 1) - (d+1) \geq  (d+1)[(x+1)^{d-1}-1] \geq 0$.
\end{proof}

\end{document}